\documentclass{article}

\usepackage{PRIMEarxiv}

\usepackage[utf8]{inputenc} 
\usepackage[T1]{fontenc}    
\usepackage{hyperref}       
\usepackage{url}            
\usepackage{booktabs}       
\usepackage{amsfonts}       
\usepackage{nicefrac}       
\usepackage{microtype}      
\usepackage{lipsum}
\usepackage{fancyhdr}       
\usepackage{graphicx}       
\usepackage{amsmath}
\usepackage{amssymb}
\usepackage{mathtools}
\usepackage{amsthm}
\usepackage{multirow}
\usepackage{tcolorbox}
\usepackage{listings}
\usepackage{xcolor}
\usepackage{makecell}
\usepackage{natbib} 
\usepackage{microtype}
\usepackage{graphicx}
\usepackage{subcaption}
\usepackage{booktabs} 
\usepackage[table]{xcolor} 
\usepackage{hyperref}

\usepackage{arydshln} 
\graphicspath{{media/}}     
\usepackage{authblk}
\usepackage{xcolor}
\usepackage{listings}
\usepackage{amsthm}
\usepackage{algorithm}
\usepackage{algorithmic} 
\newtheorem{proposition}{Proposition}
\lstdefinestyle{codestyle}{
    backgroundcolor=\color{gray!10},   
    commentstyle=\color{green!60!black},
    keywordstyle=\color{blue},
    numberstyle=\tiny\color{gray},
    stringstyle=\color{purple},
    basicstyle=\ttfamily\small,
    breaklines=true,                
    captionpos=b,                    
    keepspaces=true,                 
    numbers=left,                    
    showspaces=false,                
}

\lstdefinestyle{promptstyle}{
    basicstyle=\ttfamily\small,
    backgroundcolor=\color{black!5},
    breaklines=true,
    frame=single,
}

\title{medR: Reward Engineering for Clinical Offline Reinforcement Learning via Tri-Drive Potential Functions
}

\author[1]{\textbf{Qianyi Xu}}
\author[2]{\textbf{Gousia Habib}}
\author[1]{\textbf{Feng Wu}}
\author[3]{\textbf{Yanrui Du}}
\author[1]{\textbf{Zhihui Chen}}
\author[1]{\textbf{Swapnil Mishra}}
\author[1]{\textbf{Dilruk Perera}}
\author[1]{\textbf{Mengling Feng}}

\affil[1]{National University of Singapore, Singapore \protect\\ \texttt{xuqianyi, zhihui.chen@u.nus.edu} \protect\\
\texttt{fengwu, swapnil.mishra, dilruk, mornin@nus.edu.sg}}
\affil[2]{Finish Centre of Artificial Intelligence, University of Helsinki, Finland \protect\\ \texttt{gousia.habib@helsinki.fi}}
\affil[3]{Harbin Institute of Technology, China \protect\\ \texttt{yrdu.hit@gmail.com}}

\begin{document}
\maketitle

\begin{abstract}
  Reinforcement Learning (RL) offers a powerful framework for optimizing dynamic treatment regimes (DTRs). However, clinical RL is fundamentally bottlenecked by reward engineering: the challenge of defining signals that safely and effectively guide policy learning in complex, sparse offline environments. Existing approaches often rely on manual heuristics that fail to generalize across diverse pathologies. To address this, we propose an automated pipeline leveraging Large Language Models (LLMs) for offline reward design and verification. We formulate the reward function using potential functions consisted of three core components: survival, confidence, and competence. We further introduce quantitative metrics to rigorously evaluate and select the optimal reward structure prior to deployment. By integrating LLM-driven domain knowledge, our framework automates the design of reward functions for specific diseases while significantly enhancing the performance of the resulting policies. 
\end{abstract}

\section{Introduction}

Critical care medicine is fundamentally defined by high-stakes, sequential decision-making, RL is a natural fit for optimizing DTRs
where an agent must make sequential decisions that adapts to a patient's changing state over time. 
Reward serves as the core of RL that not only guides the agent learning but also defines the optimization problem \cite{zhuimproving}. Small changes in reward function weights can cause profound policy changes\cite{roggeveen2024reinforcement} and bad reward signals affect beyond ``win-or-lose" in gaming and can lead to ``life-or-death" in healthcare. Errors in sensitive fields such as autonomous driving have already been shown to be problematic \cite{knox2023reward}. However, reward engineering in clinical RL is less studied and benchmarked. This underscores an urgent need to design better rewards that align with clinical goals \cite{basu2025reward}.

Current rewards in RL for DTRs largely depend on manual and static design. Drawing an analogy to recent advances in LLM, we categorize existing clinical reward mechanisms into three distinct paradigms: sparse Outcome Reward Model (ORM), which assigns scalar rewards based on terminal patient outcomes like survival \cite{komorowski2018artificial}. The second is dense Process Reward Model (PRM), which depends on dense feedback based on intermediate physiological state changes \cite{zhang2024optimizing}. To mitigate the sparsity of terminal outcomes, the standard approach augments them with intermediate feedback to form a composite reward ($R_{OPRM}=R_{ORM}+R_{PRM}$) \cite{raghu2018model, lu2024reinforcement}. In this paradigm, clinicians manually select disease-relevant features and encode them using simple heuristic logic. However, the terminal reward is often weighted disproportionately high, overshadowing the dense shaping signals and the selection of specific parameters and weights remains largely ad-hoc, lacking systematic justification \cite{lu2024reinforcement, zhang2024optimizing}.

Despite efforts in designing various reward signals to accelerate policy learning, there are currently no guidelines on designing effective rewards for clinical RL due to several major challenges. First, outcomes in clinical settings are sparse and delayed. While mortality is the ultimate ground truth outcome, it occurs only at the end of ICU stay which leads to credit assignment problem. Second, the target in a treatment process is often implicit and complicated. It extends beyond binary survival to encompass the quality of recovery. There is a profound gap between the high-dimensional complexity of patient physiology and the scalar reward value required by RL algorithms. This gap leads to misaligned objectives, such as the ``Pyrrhic Victory''\cite{livingston2019pyrrhic} where an agent optimizes for survival via aggressive interventions or reward hacking. Third, rewards are not verifiable. Unlike robotics or games, clinical RL is strictly offline, precluding online exploration to validate reward signals. Consequently, current rewards rely on brittle, manual heuristics that lack interpretability and fail to generalize across different diseases or datasets. 

Recent advancements have identified LLMs as a promising solution to bridge the gap between high-level human intent and low-level reinforcement signals. While early work utilized LLMs as direct evaluators of agent behavior\cite{kwon2023reward}, more recent methods focus on using LLMs to generate executable reward function codes\cite{ma2023eureka,xie2023text2reward}. We leverage this generative capability to construct a reward function that captures the nuance of DTRs without the prohibitively high cost of manual trial-and-error engineering. Our pipeline leverages LLMs to bridge the semantic gap between medical knowledge and mathematical representation, utilizing a ``Tri-Drive'' mechanism comprising Survival, Confidence, and Competence to guide both reward generation and offline verification. Our specific contributions are:
\begin{itemize}
    \item We propose an automated pipeline for clinical RL that utilizes LLMs to perform interpretable feature selection and formulate a novel reward structure based on potential functions. 
    
    \item We design the reward functions via a Tri-Drive mechanism and derive corresponding offline quantitative metrics to act as a proxy for ground truth and verify reward functions.
    
    \item We empirically validate our framework on three distinct, high-stakes clinical tasks and demonstrate that policies trained with our automated rewards consistently outperform those using standard rewards, achieving 77.3\%, 66.7\%, and 60.3\% higher WIS scores over clinician baselines on three tasks, proving the framework's generalizability across diverse diseases.
\end{itemize}
\begin{figure*}[htbp]
    \centering
    \includegraphics[width=0.9\linewidth]{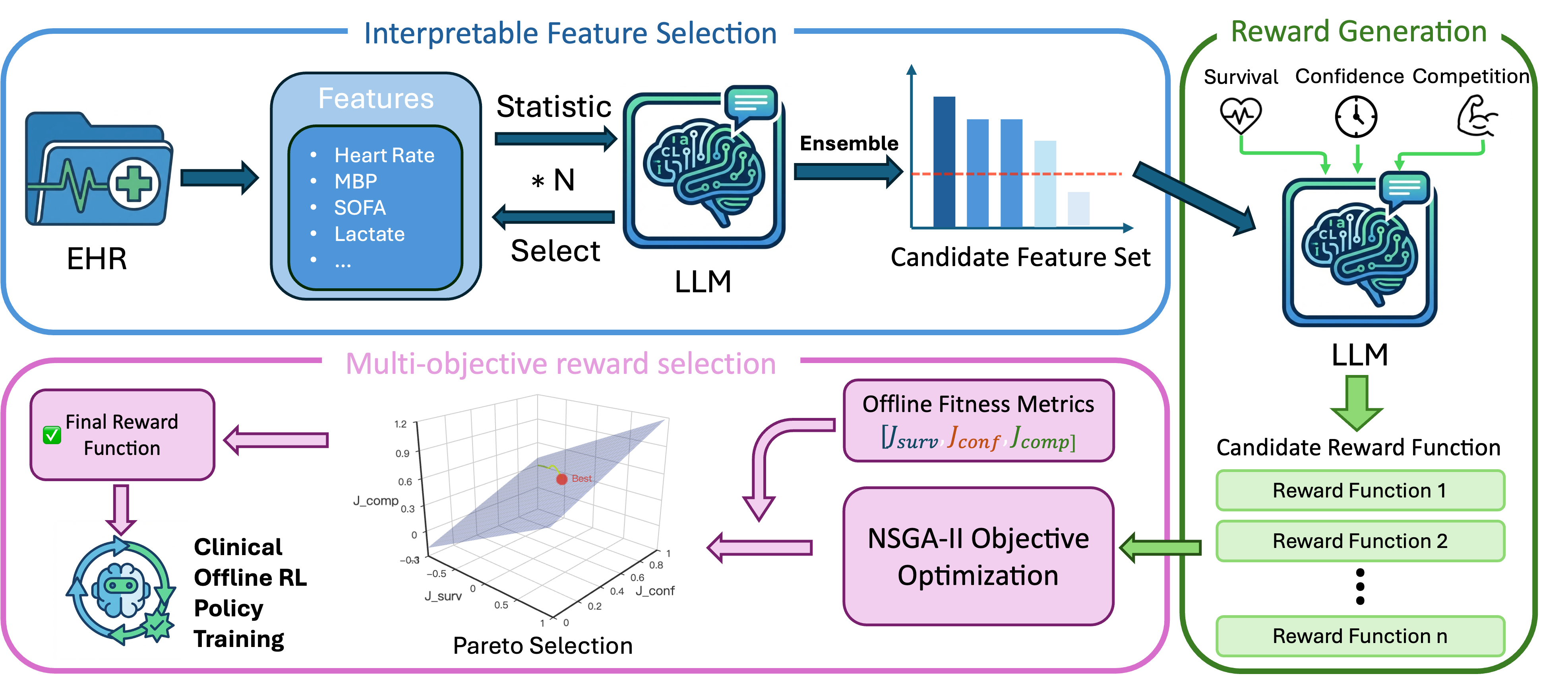} 
    \caption{Framework pipeline of medR.}
    \label{fig:framework}
\end{figure*}

\section{Background}
\label{Background}

We model the underlying clinical dynamics as a Partially Observable Markov Decision Process (POMDP), defined by the tuple $(\mathcal{S}^*, \mathcal{A}, \mathcal{T}, \mathcal{R}, \Omega, \mathcal{O}, \gamma)$. Here, $\mathcal{S}^*$ denotes the unobservable, true physiological state of the patient, and $\mathcal{A}$ represents the space of clinical interventions. The agent can not access $\mathcal{S}^*$ directly; instead, it receives intermittent observations $o_t \in \Omega$ like vital signs and lab results $\mathcal{O}(o|s^*, a)$. To ensure computational tractability for offline policy learning, we approximate this POMDP as a Markov Decision Process (MDP) defined by $(\mathcal{S}, \mathcal{A}, P, R, \gamma)$ during policy training. We explicitly construct the proxy state space $\mathcal{S}$ in our reward functions to capture temporal irregularities inherent in clinical data: a state is defined as $s_t = (o_t, \Delta t)$, where $\Delta t$ represents the time elapsed since the last observation. While standard MDP approximations often discard the uncertainty arising from partial observability, we explicitly mitigate this gap through our reward engineering framework. 

\subsection{Intrinsic Motivation}
In this work, our reward function design follows the evolutionary perspective proposed in \cite{singh2009rewards}. They distinguish between reward signals and fitness functions where reward signals are internal scalar values that drive learning within an agent and fitness is the external evaluation that measures an agent's success. For bounded agents in an offline environment, the optimal reward function is not the fitness function\cite{singh2010intrinsically}. The design of reward functions should optimize intrinsic motivations (IM), unlike previous literature about IM that focuses on incentivizing exploration which is dangerous in our offline setting, in our paper, we form it as a threefold drive: the drive to restore homeostasis(survival), the drive to lower uncertainty(confidence), and the drive to minimize cost(competence). In Hull's Drive Reduction Theory\cite{hull1945principles}, organisms are motivated to maintain physiological equilibrium, we treat homeostasis as a proxy of survival to optimize in the reward function. Unlike curiosity-driven frameworks that encourage exploration under uncertainty\cite{berlyne1960conflict,schmidhuber1991possibility}, in order to design a reward function specific to our offline environment\cite{singh2009rewards}, we conversely suppress uncertainty due to the sensitivity in critical care\cite{djulbegovic2011uncertainty}. \citet{white1959motivation} furthers the learning process of effective interaction with the environment, which is transferrable as effective treatment under clinical setting.

\subsection{Potential-Based Reward}
Reward shaping was used in previous literature on clinical RL to address the challenge of sparse feedback\cite{den2024guideline}. To ensure that this augmentation does not alter the optimal policy $\pi^*$, \citet{ng1999policy} formalized Potential-Based Reward Shaping (PBRS). They proved that if the shaping term is defined as the difference of a potential function $\Phi(s)$ over states:$F(s, a, s') = \gamma \Phi(s') - \Phi(s)$, then the optimal policy under the shaped reward $R' = R + F$ is identical to that under $R$. Later extensions generalized this framework to state-action potentials~\cite{wiewiora2003principled} and dynamic, time-dependent potentials $\Phi(s,t)$~\cite{devlin2012dynamic}. Most recently, \citet{forbes2024potential} applied this theory to intrinsic motivation (PBIM), utilizing potentials to mathematically neutralize the policy-altering bias of exploration bonuses such as curiosity. However, the strict policy invariance mandated by PBRS relies on the assumption that the base reward signal which in our case is the mortality already defines the ideal behavior. Here we keep the form of potential functions as indicators for patient health transients and transform the problem into potential-based reward engineering.

\section{Methodology}
\subsection{Interpretable Feature Selection}
We leverage the semantic reasoning capabilities of LLMs to select critical features for constructing the reward functions. Figure\ref{fig:framework} shows the framework of our method. We first compute a statistical summary for all candidate features in the dataset $\mathcal{D}$. Let $f$ be a candidate feature and $y$ be the patient outcome mortality. We generate a metadata tuple $\mathcal{M}_f = (\mu_f, \sigma_f, \eta_f,\rho_f^{outcome}, \rho_f^{action})$, where $\mu$ is the mean and $\sigma$ is the standard deviation of the non-missing values. We explicitly analyze the missingness rate $\eta$ because Electronic Health Records(EHRs) are inherently sparse, simply using imputed data for reward computation is misleading and can introduce bias. $\rho_{outcome}, \rho_{action}$ are the correlation between feature $x$ and patient outcome $y$ or action $a$, for each mortality and action dimension, we compute the Pearson correlation coefficient with every feature $f$ in the state space.

We create a structured prompt that feeds this statistical metadata $\mathcal{M}$ into an LLM. The LLM acts as a clinical data scientist to select $K$ critical state features that affect the specific disease treatment. Crucially, for every selected feature, the model must generate a rationale explaining its selection. To mitigate the hallucination and instability inherent in LLM generation, we employ a stochastic ensemble selection strategy. The final feature set $\mathcal{F}^*$ is determined via majority voting:
$\mathcal{F}^* = \{ f \in \mathcal{F}_{all} \mid \frac{1}{N} \sum_{i=1}^N \mathbb{I}(f \in \mathcal{S}_i) \ge \tau \}$ where $\mathcal{S}_i$ is the set of features selected in iteration $i$, and $\kappa$ is the consensus threshold. We then double check with clinicians on the features chosen by LLM and the rationales behind them. In this way we effectively simulate the knowledge of multiple clinicians to vote on the most robust candidates and validate with real clinicians to guarantee feature reliability. 

\subsection{Potential Based Reward Function Generation}
We design a health potential that reflects the quality of physiological states. Past reward functions often define specific rules that directly scores the action, for example, Reward $+1$ if $70<mbp<80$, which will likely cause reward hacking and encouraging actions that can directly increase blood pressure. Instead, we form it as the gradient of potential to capture the patient state change. To bridge the gap between sparse outcome labels and dense clinical supervision, the reward function $R$ is composed of three distinct drives—Survival, Confidence, and Competence leveraging the intrinsic motivation theory \cite{singh2009rewards}. We define the reward function as follows:
\begin{equation}
\small
R(s, a, t, s', a', t') = \left[ \gamma \Phi(s_{t+1}, t+1) - \Phi(s_t, t) \right] - \lambda \mathcal{C}(a_t)
\end{equation}
\begin{equation} 
\small
    \Phi(s_t, t) = \delta(t) \cdot \sum_{f \in \mathcal{F}} \omega_f \cdot \mathcal{S}_f(v_{t,f}) \cdot \mathcal{U}_f(\Delta t_{t,f})
\end{equation}
where $\gamma$ is the discount factor, $\Phi(\cdot)$ is the health potential, $\delta(t)$ is the strategic decay, $\mathcal{S}_k(v_{t,f})$ is the survival score with confidence weight $\mathcal{U}_f(\Delta t_{t,f})$ and $\mathcal{C}(\cdot)$ is the competence cost.

We use physiological stabilization as the proxy for survival. By embedding the metadata of critical features as part of the prompt, we personalize the reward design to handle critical patients with comorbidities. The LLM is prompted to design a potential function that rewards states closer to homeostasis. In the offline clinical setting, data irregularity poses a significant risk: a normal heart rate recorded 12 hours ago is not a reliable indicator of current stability. Since we are unable to refresh the data by ordering new lab tests, we utilize the elapsed time $\Delta t$ as indicators of freshness of the data $s_t = (o_t, \Delta t_t)$. The LLM generates confidence weights for each critical feature that inversely scales the reward based on $\Delta_{t,f}$. This mechanism penalizes stale states, effectively discounting the survival reward when data is uncertain to reduce epistemic uncertainty. We also add a strategic decay to prevent unnecessarily longer ICU stay. To prevent the agent from achieving survival via excessive or harmful interventions, we incorporate a competence component focused on treatment efficiency. Mathematically, this serves as a regularization term: if two treatments result in the same physiological transition $s_t \rightarrow s_{t+1}$, the potential function assigns a higher value to the less invasive action, penalizing over-treatment or flailing behavior. By decoupling the accumulating action cost $\mathcal{C}(a)$ from the telescoping potential $\Phi(s, t)$\ref{prop:non_telescoping}, we effectively formulate the RL objective as the Lagrangian of a Constrained Markov Decision Process (CMDP). The agent maximizes the Lagrangian $\mathcal{L} = \text{Physiology} - \lambda \cdot \text{Toxicity}$, ensuring that high-stability states are only pursued if the cost of the intervention path does not exceed the value of the stabilization\ref{prop:lagrangian_equivalence}. Instead of manual heuristic tuning, the LLM acts as the logic designer. We feed the definitions of these three drives along with the feature metadata into the LLM, which autonomously synthesizes the logic into executable Python code for the reward function. To mitigate the risk of hallucination and capture diverse clinical perspectives, we generate $N$ distinct candidate functions in parallel, creating a diverse pool of candidates for the subsequent verification phase.

\subsection{Tri-Drive Reward Selection}
Since online verification is ethically prohibitive in clinical settings, we cannot evaluate reward candidates by deploying them on patients. Instead, we propose a set of offline fitness metrics to rigorously assess the quality of each generated reward function $R$. We formulate reward design as a Tri-Drive Selection Problem aiming to balance three conflicting objectives based on trajectory-level labeled ground truth: Survival, Confidence, and Competence Fitness.Survival Fitness ($J_{surv}$) validates whether the induced reward signal correlates with actual patient outcomes. We define a ground-truth trajectory score $G(\tau)$ that consists of both mortality and Sequential Organ Failure Assessment(SOFA) score\cite{lambden2019sofa} increase compared to baseline(time when admitted to ICU). The fitness is the correlation $\rho$ between the cumulative reward and $G(\tau)$:
\begin{equation}
G(\tau) = \mathbf{1}(\text{Survival}) + \frac{1}{T} \sum_{t=0}^{T-1} \mathbf{1}(|\text{SOFA}_t - \text{SOFA}_{baseline}| < \epsilon)
\end{equation}
\begin{equation}
J_{surv}(R) = \rho\left(R_{\tau}, G(\tau)\right)
\end{equation}
Confidence Fitness ($J_{conf}$) penalizes reward functions that assign high credit to decisions made under high epistemic uncertainty. We define an uncertainty proxy $U(\tau)$ as the average time-gap $\Delta t$ of all the features $f$ in the critical feature set $F$. The fitness is the negative correlation between the cumulative reward and trajectory uncertainty:
\begin{equation}
    U(\tau) = \frac{1}{T \cdot F} \sum_{t=0}^{T-1} \sum_{f=1}^{F} \Delta t_{t,f}.
\end{equation}
\begin{equation}
J_{conf}(R) = -\rho\left(R_{\tau}, U(\tau)\right)
\end{equation}
Competence Fitness ($J_{comp}$) ensures that the reward function prioritizes efficient strategies that achieve or maintain high physiological stability with minimal necessary intervention. We define feature type as $\mathcal{T}=\{N, D_l,D_h\}$ where $N$ denotes Normal-Range features, and $D_l, D_h$ denote Directional features where minimization or maximization is physiologically optimal, respectively. We define the unified homeostasis score $H(x)\in [0,1]$ for a feature of type $T \in \mathcal{T}$ as:
\begin{equation}
\small
    h(f) = 
    \begin{cases} 
        \mathbf{1}_{f \in \mathcal{I}} + \mathbf{1}_{f \notin \mathcal{I}} \cdot \sigma\left( k \left[ 0.5 - \frac{d(f, \mathcal{I})}{\text{IQR}} \right] \right)
        & \text{if } T = N, \\
        
        \sigma\left( k \left[ 0.5 - f \right] \right)
        & \text{if } T = D_l, \\

        \sigma\left( -k \left[ 0.5 - f \right] \right)
        & \text{if } T = D_h,
    \end{cases}
\end{equation}
\begin{equation}
    H(s_t) = \frac{1}{|F|} \sum_{f \in F} h(_{t,f}).
\end{equation}
We then define the Efficiency $E_t$ as the homeostasis increase penalized by the magnitude of the dose taken:
\begin{equation}
    E_t = H(s_{t+1}) - H(s_t) - \alpha \cdot \bar{a}_t.
\end{equation}
\begin{equation}
    J_{comp} = \rho\left(R_\tau, E_\tau\right),
\end{equation}
where $\sigma(z) = (1 + e^{-z})^{-1}$ is the standard sigmoid function, $d(f, \mathcal{I})$ is the distance from feature $f$ to the interval $\mathcal{I}$, $\text{IQR}$ is the interquartile range width used for normalization. $\alpha$ and $k$ are cost parameters, $\bar{a}_t$ is the magnitude of the intervention.
We employ the Non-Dominated Sorting Genetic Algorithm II (NSGA-II) \cite{deb2002fast} to rank the generated candidates based on these three objectives. Instead of selecting a single weighted average, we identify the Pareto Frontier where no objective can be improved without degrading another. From this frontier, the solution closest to the utopia point representing the best trade-off.

\section{Experiments}
\subsection{Tasks}
We evaluate our framework on three distinct critical care tasks, covering different datasets, demographics, and action space types.\\
\textbf{Sepsis Treatment}
We utilize the Medical Information Mart for Intensive Care (MIMIC)-IV v3.1\cite{johnson2020mimic} database to learn optimal Intravenous (IV) fluid and Vasopressor administration strategies for sepsis patients. We select adult patients (age $\ge$ 18) satisfying the Sepsis-3 criteria \cite{singer2016third}, and exclude patients with missing mortality, demographics, and interventions, extreme missingness ($>$60\% missing data) or short hospital stays ($<$24 hours). We define a septic shock starting from 24 hours before diagnosis and ending at 48 hours post-diagnosis. After filtration, 7,501 patients with 46 features are retained in the dataset. The action space is discrete consisting of two primary interventions: IV fluids and Vasopressors. We discretize IV fluid volume and vasopressor dosage into 5 bins each, resulting in a combined action space of $5 \times 5 = 25$ discrete actions. \\
\textbf{Mechanical Ventilation}
We use the multi-center eICU Collaborative Research Database (eICU-CRD)\cite{pollard2018eicu} to optimize ventilator settings for patients with respiratory failure. The cohort includes adult patients who underwent invasive mechanical ventilation (MV) for at least 24 hours. We select patients aged \(\geq\)16 who received MV for at least 24 hours. We also exclude patients with missing mortality outcomes, demographics, interventions, and those ventilated for \(>\)14 days. The final cohort contains 2168 patients with 41 features. The action space is discrete, controlling three key ventilator parameters: Positive End-Expiratory Pressure (PEEP), Fraction of Inspired Oxygen (FiO$_2$), and Tidal Volume. We define a discrete action set based on clinically relevant adjustments for each parameter, resulting in a composite discrete action space of 18 actions (PEEP: 2 levels, FiO$_2$: 3 levels, Tidal Volume: 3 levels).\\
\textbf{Renal Replacement Therapy}
We apply our method to the AmsterdamUMCdb \cite{thoral2021sharing} to optimize continuous renal replacement therapy (RRT) dosing. We select adult patients (age $\ge$ 18) who received CRRT during the ICU stay where the stay is no longer than 60 days. There are 857 patients with 18 features in the cohort. Unlike the previous tasks, the action space here is continuous. The agent controls the effluent flow rate clipped to the range $[0, 60]$ ml/kg/h, a critical parameter determining the intensity of dialysis.

For all tasks, clinical time-series data (vitals, labs, medications) are aggregated into 1-hour time steps. Missing values are imputed using forward-filling for time-varying features. For CRRT effluent rate, we imputed values based on RRT sessions, defining a new session whenever the gap between two dialysis records exceeded 8 hours. All continuous state features are normalized to zero mean and unit variance to ensure training stability. For policy training, we use common safe offline RL methods, Batch-Constrained Q-Learning (BCQ)\cite{fujimoto2019off} for discrete action space tasks (Sepsis and Ventilation), and Implicit Q-Learning (IQL)\cite{kostrikov2021offline} for continuous action space task (RRT).

\subsection{Baselines}
To rigorously evaluate the effectiveness of our automated reward design framework, we compare it against two categories of baselines: manually engineered heuristic rewards and direct LLM approaches.

\subsubsection{Heuristic Reward Design}
We implement three standard reward structures commonly used in clinical RL literature. These serve as the "human-expert" baselines, representing varying degrees of clinical foresight:\\
\textbf{Outcome Reward Model (ORM):} A purely outcome-oriented function that assigns a non-zero reward only at the terminal step, $+100$ for survival (discharge) and $-100$ for mortality. The value of $\pm100$ is used in both Sepsis\cite{komorowski2018artificial} and Ventilation\cite{peine2021development} tasks in previous literature. Since there is no past work that utilized outcome reward for RRT task, we adopt the same setting of $\pm100$ for RRT. This baseline tests the agent's ability to solve the credit assignment problem without intermediate guidance. We explicitly utilized ICU-discharge mortality rather than 90-day mortality as our terminal state definition in the baselines. We argue that an RL agent can only be held accountable for outcomes within its scope of control. An agent penalized for a death occurring 60 days post-discharge may erroneously learn that stabilizing and discharging a patient is suboptimal, leading to conservative strategies. \\
\textbf{Process Reward Model (PRM):} A process-oriented function that provides immediate feedback based on step-wise improvements in physiological stability. It incentivizes short-term stabilization without long-term survival awareness. We use the reward function fully based on SOFA score\cite{zhang2024optimizing} for Sepsis. For Ventilation, though we note that there is one process reward based on ventilator free days\cite{yousuf2025intellilung}, the reward calculated on our cohort is not dense enough for efficient policy training, we therefore adopt the intermediate reward from the OPRM\cite{liu2024reinforcement}. This further indicates that a fixed reward function may not be transferrable between different cohorts even with the same dataset. For RRT, we use the reward function based on Estimated Glomerular Filtration Rate (eGFR)\cite{zhang2024reinforcement}.\\
\textbf{Outcome Process Reward Model(OPRM):} A weighted combination of the sparse and dense signals. This represents the state-of-the-art manual design, aiming to balance immediate physiological regularization with the ultimate goal of patient survival. We select from current SOTA methods \cite{perera2026smart}for Sepsis and \cite{liu2024reinforcement}for Ventilation. Due to the lack of work on using OPRM for RRT, we choose to combine the PRM(eGFR) with ORM($\pm100$) as OPRM.

\subsubsection{LLM-based Design}
We also compare against methods that leverage LLMs for reward generation without our proposed iterative refinement process. Since there are limited methods that leverage LLM to generate rewards for offline RL, we design two variants:\\
\textbf{LLM-as-Reward (LLMR):} In this setting, the LLM is used as a direct proxy for the reward function or a scorer of the clinicians' actions. We embed the whole patient trajectory $(s,a,s')$ with final outcome mortality as part of the prompt along with clinical context, based on the mortality of the patient, we let the LLM do credit assignment of $+15$ for survival or $-15$ for mortality along the trajectory. The LLM will choose steps that it thinks is the most crucial for the decision-making and the rest of the steps will be assigned 0. After the reward generation we will make sure the sum of reward is $+15$ or $-15$. This tests the LLM's innate ability to evaluate clinical actions without explicit function coding.\\
\textbf{Code Generation (CodeGen):} The LLM is prompted to write a Python reward function in a single pass, given only the task description and list of available features without any further instructions on the reward components. This baseline evaluates the code-generation capability of the LLM absent the Tri-Drive and potential functions utilized in our method.
\subsection{Evaluation Metrics}
We employ a comprehensive set of metrics to evaluate both the intrinsic quality of the generated reward functions and the extrinsic performance of the resulting RL policies.\\
\textbf{Action Agreement among Survivors:} We calculate the percentage of time steps among the whole trajectory where the RL agent's chosen action matches the clinician's action, specifically within the cohort of patients who survived. High agreement on survivors suggests the agent learns safe, effective behaviors from successful clinical examples.\\
\textbf{Off-Policy Evaluation(OPE)}
To quantify the policy's theoretical performance, we employ Weighted Importance Sampling (WIS), a standard unbiased estimator for off-policy value estimation. To facilitate direct comparison, we use the same rewards generated by DeepSeek-R1 that was not involved in training nor as a baseline for WIS calculation on all the policies. We also plot WIS training curve to make sure that the return is higher as we train the policy.\\
\textbf{Mortality vs. Cumulative Reward:} To prove that our reward is highly correlated with ground truth mortality, we plot the observed patient mortality rate against cumulative reward along patient trajectories assigned by the clinician's policy. A monotonic decreasing trend indicates that the learned value function correctly identifies patient risk. We compare with PRM which is the only baseline that also does not contain mortality as part of rewards.

\begin{table*}[htbp]
\centering
\caption{Agreement rates among survivors compared to clinician actions across tasks.}
\label{tab:agreement}
\resizebox{\textwidth}{!}{
\begin{tabular}{lcccccccccc}
\toprule
\textbf{Method} 
& \multicolumn{3}{c}{\textbf{Sepsis}} 
& \multicolumn{4}{c}{\textbf{Ventilation}} 
& \multicolumn{2}{c}{\textbf{RRT}} \\
\cmidrule(lr){2-4} \cmidrule(lr){5-8} \cmidrule(lr){9-10}
& Joint & IV Fluids & Vasopressors
& Joint & PEEP & FiO$_2$ & Tidal Volume
& Dose & On/Off \\
\midrule
ORM 
& $25.37 \pm 0.78$ & $30.01 \pm 0.69$ & $77.98 \pm 1.01$
& $19.79 \pm 3.13$ & $74.80 \pm 3.13$ & $47.87 \pm 2.26$ & $48.23 \pm 3.41$
& $50.77 \pm 15.96$ & $67.54 \pm 16.14$ \\

PRM 
& $27.54 \pm 1.03$ & $33.17 \pm 0.96$ & $78.47 \pm 0.95$
& $21.24 \pm 2.37$ & $74.49 \pm 3.26$ & $49.34 \pm 2.60$ & $50.70 \pm 3.51$
& $54.64 \pm 10.51$ & $74.45 \pm 3.03$ \\

OPRM 
& $27.51 \pm 1.39$ & $33.04 \pm 1.27$ & $78.44 \pm 1.08$
& $21.25 \pm 2.58$ & $73.24 \pm 3.63$ & $49.30 \pm 2.22$ & $50.02 \pm 2.88$
& $56.46 \pm 11.66$ & $74.83 \pm 2.46$ \\
\hdashline \noalign{\vskip 2pt}
LLMR$_{\texttt{+GPT-OSS-20B}}$ 
& $24.78 \pm 0.48$ & $30.46 \pm 0.48$ & $77.62 \pm 1.05$
& $21.86 \pm 5.85$ & $75.33 \pm 5.59$ & $52.39 \pm 5.86$ & $46.08 \pm 7.39$
& $50.63 \pm 15.14$ & $72.95 \pm 5.04$ \\
CodeGen$_{\texttt{+GPT-5}}$ 
& $27.63 \pm 0.87$ & $33.21 \pm 0.61$ & $78.46 \pm 1.00$
& $21.09 \pm 2.21$ & $76.04 \pm 3.34$ & $49.30 \pm 2.06$ & $49.18 \pm 2.90$
& $54.15 \pm 11.38$ & $74.43 \pm 3.00$ \\
CodeGen$_{\texttt{+Qwen3-Max}}$ 
& $27.55 \pm 0.47$ & $33.41 \pm 0.28$ & $78.56 \pm 0.99$
& $21.31 \pm 2.85$ & $76.48 \pm 3.52$ & $49.78 \pm 2.59$ & $50.72 \pm 4.25$
& $53.82 \pm 11.77$ & $70.46 \pm 9.84$ \\
\midrule
medR$_{\texttt{+GPT-OSS-20B}}$ 
& $27.45 \pm 0.93$ & $33.67 \pm 0.59$ & $78.46 \pm 1.07$
& $24.10 \pm 4.40$ & $\mathbf{77.06 \pm 4.29}$ & $\mathbf{57.13 \pm 3.30}$ & $\mathbf{50.83 \pm 3.65}$
& $53.82 \pm 11.77$ & $70.46 \pm 9.84$ \\
\textbf{medR$_{\texttt{+GPT-4}}$}
& $\mathbf{27.70 \pm 0.41}$ & $\mathbf{33.75 \pm 0.22}$ & $\mathbf{78.47 \pm 0.95}$
& $\mathbf{24.20 \pm 4.23}$ & $76.67 \pm 3.95$ & $56.92 \pm 3.79$ & $50.41 \pm 3.59$ & $\mathbf{66.44 \pm 8.42}$ & $\mathbf{76.92 \pm 4.30}$ \\
\bottomrule
\end{tabular}%
}
\end{table*}

\subsection{RQ1: Should we use mortality as fitness or reward?}
A prevalent assumption in clinical RL is that agents trained solely on sparse mortality outcomes will asymptotically converge to optimal behavior via automatic temporal credit assignment. However, this premise rarely holds true in the offline, high-noise, stochastic regime of critical care, and can lead to ``superstitious learning.'' In ICU a robust patient may survive despite suboptimal actions, while a high-severity patient may die despite optimal care. As shown in Table\ref{tab:agreement} and Table\ref{tab:action_agreement}, the policy trained under ORM performs significantly worse than the other baseline and our method  on the agreement among survivors which reveals that the former struggles to distinguish between policy quality and patient stochasticity. Simply maximizing survival may move the patient away from wellbeing. In Table~\ref{tab:wis}, the ORM reward also results in the lowest return in both Sepsis and Ventilation tasks. Furthermore, utilizing outcome that happens later can introduce hindsight bias, the training signal should only depend on information available at decision time. We avoid this by aligning the reward with the clinician's reality, optimizing immediate physiology as a proxy of survival. We separate mortality from reward and instead use the hindsight ground truth as fitness to evaluate whether the reward function induces high-quality policies. This confirms that explicitly rewarding the process of stabilization provides a more reliable learning signal than the noisy outcome.

\begin{figure}[htbp]
    \centering
    \includegraphics[width=\linewidth]{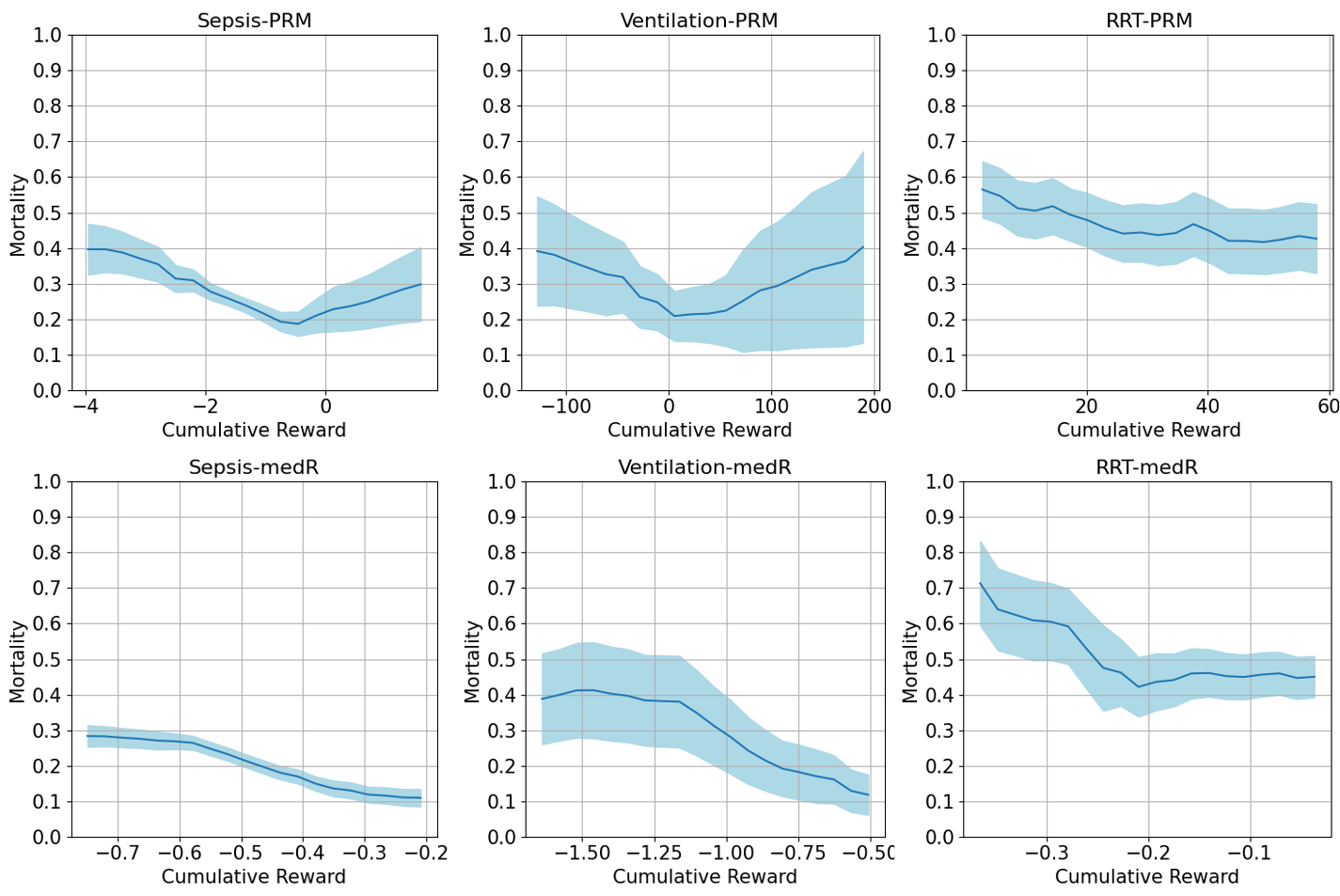} 
    \caption{The relationship between cumulative process reward and patient mortality probability for different tasks. A valid process should demonstrate clear downward trend with mortality, baselines that include mortality as part of the reward design are excluded for fair comparison.}
    \label{fig:mortality_analysis}
\end{figure}

\subsection{RQ2: Can we use naive shaping of outcome reward model?}
To facilitate policy training previous literature actively utilizes clinically-guided dense rewards that include both terminal and intermediate rewards to form an OPRM\cite{raghu2017continuous, lu2024reinforcement}, which can be approximated as a naive reward shaping in a broader sense. As shown in Figure\ref{fig:mortality_analysis}, the standalone process reward is not well correlated with mortality, which is dangerous as a secondary outcome to guide policy learning. In Figure\ref{fig:j_bar}, naive shaping candidates often dominate on the $J_{surv}$ axis but collapse on the $J_{comp}$ axis. This indicates a reward hacking phenomenon where the agent learns to accumulate stability bonuses by maintaining the patient in a safe but non-improving state, rather than efficiently steering them toward discharge. This is proved by Table\ref{tab:wis} where it achieves low WIS in both Ventilation and RRT tasks. In contrast, medR maintains high performance across all three objectives. Our method avoids this failure mode through the time-dependent decay factor $\delta(t)$ explicitly penalizes loitering. As $\Phi(s,t)$ decays towards zero over time, the agent is mathematically forced to achieve stabilization targets early in the episode to maximize the return. By utilizing the potential difference $\gamma \Phi_{t+1} - \Phi_t$, our method mathematically ensures that rewards are only generated by state improvement rather than state maintenance, effectively preventing the infinite horizon hacking observed in the baseline.

\begin{table}[h]
\centering
\caption{Comparison of WIS scores with bootstrapped confidence intervals (95\%CI). }
\label{tab:wis}

\begin{tabular}{lccc}
\toprule
\textbf{Method} & \textbf{Sepsis} & \textbf{Ventilation} & \textbf{RRT} \\
\midrule

Clinician (Baseline) 
& \makecell{$-54.50$ \\ \footnotesize{$(-55.49, -53.51)$}} 
& \makecell{$-56.94$ \\ \footnotesize{$(-57.29, -56.59)$}} 
& \makecell{$-87.61$ \\ \footnotesize{$(-92.51, -82.83)$}} \\
\midrule

ORM 
& \makecell{$-45.34$ \\ \footnotesize{$(-46.53, -44.02)$}} 
& \makecell{$-31.06$ \\ \footnotesize{$(-32.90, -29.20)$}} 
& \makecell{$-34.95$ \\ \footnotesize{$(-51.31, -17.35)$}} \\

PRM 
& \makecell{$-13.12$ \\ \footnotesize{$(-14.03, -12.24)$}} 
& \makecell{$-24.97$ \\ \footnotesize{$(-26.47, -23.59)$}} 
& \makecell{$-38.97$ \\ \footnotesize{$(-51.82, -26.13)$}} \\

OPRM 
& \makecell{$-17.73$ \\ \footnotesize{$(-18.93, -16.63)$}} 
& \makecell{$-24.30$ \\ \footnotesize{$(-25.59, -23.20)$}} 
& \makecell{$-35.57$ \\ \footnotesize{$(-51.75, -19.90)$}} \\
\midrule

LLMR$_{\texttt{+GPT-OSS-20B}}$ 
& \makecell{$-17.86$ \\ \footnotesize{$(-18.88, -17.11)$}} 
& \makecell{$-20.72$ \\ \footnotesize{$(-22.89, -18.82)$}} 
& \makecell{$-38.98$ \\ \footnotesize{$(-52.63, -25.34)$}} \\

CodeGen$_{\texttt{+GPT-5}}$
& \makecell{$-28.68$ \\ \footnotesize{$(-31.08, -25.75)$}} 
& \makecell{$-21.97$ \\ \footnotesize{$(-24.73, -19.34)$}} 
& \makecell{$-39.90$ \\ \footnotesize{$(-53.42, -26.37)$}} \\

CodeGen$_{\texttt{+Qwen3-Max}}$ 
& \makecell{$-15.17$ \\ \footnotesize{$(-15.82, -14.52)$}} 
& \makecell{$-24.06$ \\ \footnotesize{$(-26.98, -22.04)$}} 
& \makecell{$-41.15$ \\ \footnotesize{$(-53.77, -27.98)$}} \\
\midrule

\textbf{medR$_{\texttt{+GPT-OSS-20B}}$} 
& \makecell{$-12.44$ \\ \footnotesize{$(-13.23, -11.64)$}} 
& \makecell{$\mathbf{-18.42}$ \\ \footnotesize{$(-20.13, -16.71)$}} 
& \makecell{$-39.49$ \\ \footnotesize{$(-52.37, -26.61)$}} \\

\textbf{medR$_{\texttt{+GPT-4}}$} 
& \makecell{$\mathbf{-12.37}$ \\ \footnotesize{$(-13.12, -11.63)$}} 
& \makecell{$-18.96$ \\ \footnotesize{$(-20.77, -17.13)$}}
& \makecell{$\mathbf{-34.73}$ \\ \footnotesize{$(-52.70, -16.53)$}} \\

\bottomrule
\end{tabular}%
\end{table}
\subsection{RQ3: Do we need to include action as part of reward?}
Previous reward function design rarely includes action as part of reward signals. Without action costs, clinical agents are prone to masking physiological deterioration with high-intensity interventions. While this improves the immediate state representation, it often induces long-term toxicity. In Figure\ref{fig:sepsis_agreement},\ref{fig:ventilation_agreement},\ref{fig:rrt_agreement}, policies trained by heuristic reward tend to give higher dosages compared to medR. By strictly penalizing the intervention magnitude, our full Tri-Drive objective forces the agent to solve a Lagrangian relaxation of the clinical problem: maximize stability subject to minimum necessary harm. Our policy achieves the highest $J_{comp}$ in all three tasks (Figure\ref{fig:j_bar}) while achieving the highest WIS(Figure\ref{tab:wis}).
\subsection{RQ4: Do LLMs qualify as reward functions?}
\label{subsec:rq4_llm_as_reward}
LLMs are rapidly developing and demonstrating excellent reasoning capabilities. Previous research has explored the possibility of directly using LLMs as planners for DTRs\cite{luo2025large} and concludes that they exhibit notable limitations. Therefore, a natural question is whether they can instead directly serve as the reward signal instead and output a scalar quality score. Our experiments indicate that LLMs are ill-suited for this direct role in the clinical offline RL setting. While an LLM can accurately identify immediate physiological abnormalities, it fails to perform temporal credit assignment. The standard LLMR baseline consistently demonstrates lower agreement with clinical policies compared to medR across all tasks, particularly in complex interventions like RRT dosing where it trails by over 15 percentage points ($50.63\%$ vs. $66.44\%$). Furthermore, LLMR exhibits notably higher variance, suggesting that generic LLM-generated rewards lack the stability and precision of our potential-based approach. While they naturally demonstrate comparably higher $J_{surv}$ since we force the sum of reward to be a direct indicator of mortality, they systematically collapse on $J_{conf}$ and $J_{comp}$, suggesting it generates survival-at-all-costs policies that ignore clinical efficiency and data uncertainty.
By prompting the LLM with the full trajectory and outcome, the hindsight bias appears again and the model will tend to hallucinate positive or negative rewards along the whole trajectory. This confirms that while LLMs possess domain knowledge, they function as semantic pattern matchers rather than value function approximators.
\begin{table}[t]
\centering
\caption{Ablation study. We evaluate the contribution of feature selection, the potential term, and Tri-Drive logic and prove medR demonstrates the most consistent performance across tasks.}
\label{tab:ablation_study}
\begin{tabular}{l c c c}
\toprule
\textbf{Ablation Variant} & $J_{\text{surv}}$ & $J_{\text{conf}}$ & $J_{\text{comp}}$ \\
\midrule

\multicolumn{4}{c}{\textit{Sepsis}} \\
\midrule
w/o Feature Selection & 0.61 & 0.47 & 0.55 \\
w/o Potential & 0.44 & 0.50 & 0.65 \\
w/o Tri-Drive Logic & 0.56 & 0.54 & 0.56 \\
\textbf{medR} & \textbf{0.68} & \textbf{0.57} & \textbf{0.69} \\

\midrule
\multicolumn{4}{c}{\textit{Ventilation}} \\
\midrule
w/o Feature Selection & 0.65 & 0.77 & 0.55 \\
w/o Potential & 0.41 & 0.24 & 0.49 \\
w/o Tri-Drive Logic & 0.60 & 0.74 & 0.59 \\
\textbf{medR} & \textbf{0.69} & \textbf{0.82} & \textbf{0.61} \\

\midrule
\multicolumn{4}{c}{\textit{RRT}} \\
\midrule
w/o Feature Selection & 0.50 & 0.46 & 0.42 \\
w/o Potential & 0.54 & 0.32 & 0.48 \\
w/o Tri-Drive Logic & 0.52 & \textbf{0.70} & 0.58 \\
\textbf{medR} & \textbf{0.60} & 0.66 & \textbf{0.60} \\

\bottomrule
\end{tabular}%
\end{table}

\subsection{RQ5: Can LLMs directly generate reward code?}
\label{subsec:rq5_direct_generation}
If LLMs cannot directly act as the reward model, can they write the codes for a high-quality reward function directly? We evaluated the capability of multiple LLMs to generate reward functions from open-ended clinical prompts. The results were consistently suboptimal. Table~\ref{tab:agreement} reveals a critical limitation in direct code generation approaches (CodeGen) where they lack the structural robustness required for complex interventions like RRT. They also yield lower WIS(Figure\ref{tab:wis} where our potential-based medR framework outperforms the strongest CodeGen baseline. These results underscore that raw LLM reasoning capability alone is insufficient; without the control-theoretic scaffolding provided by medR, even superior models struggle to synthesize reward functions that capture the nuance of rare clinical events. Moreover, CodeGen variants exhibit significant volatility; CodeGen$_{\texttt{+GPT-5}}$(CG(G)) achieves a high $J_{surv}$ of $0.73$ in Sepsis but drops to $0.50$ in Competence, indicating the agent maximizes the outcome signal via inefficient or unsafe intervention loads. In contrast, medR maintains a robust fitness profile across all three dimensions, validating that our potential function effectively balances survival objectives with necessary safety and efficiency constraints.

\begin{figure}[htbp]
    \centering
    \includegraphics[width=\linewidth]{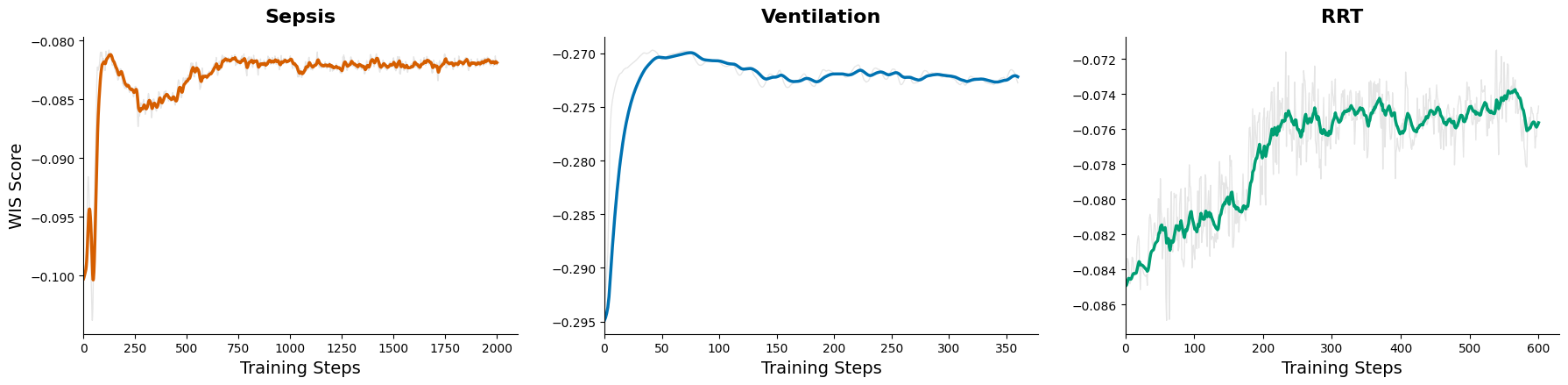} 
    \caption{WIS training plot.}
    \label{fig:wis_train}
\end{figure}

\subsection{Ablation Study}
\label{subsec:ablation}
To rigorously evaluate the individual contributions of our framework's components, we conducted an ablation study by systematically removing (1) the LLM-driven Feature Selection, (2) the Potential-based reward formulation, and (3) the Tri-Drive structured logic. Table~\ref{tab:ablation_study} summarizes the impact of these components on policy fitness ($J_{\text{surv}}$, $J_{\text{conf}}$, $J_{\text{comp}}$) across all three clinical domains. The most critical component of our framework is the potential difference formulation ($\gamma \Phi_{t+1} - \Phi_t$). Removing this term results in a catastrophic collapse in performance, particularly in the Mechanical Ventilation domain, where $J_{\text{conf}}$ drops precipitously from $0.82$ to $0.24$. This confirms that without the telescoping sum property, the reward signal fails to distinguish between state maintenance and state improvement, leading to policies that are highly uncertain and poorly correlated with patient survival. The ablation of automated feature selection highlights the necessity of filtering the high-dimensional clinical state space. This effect is most pronounced in the RRT task, where removing feature selection degrades survival fitness ($J_{\text{surv}}$) from $0.60$ to $0.40$. In complex domains like renal replacement, irrelevant features introduce significant noise; our LLM-driven selection effectively isolates the causal variables required for stable learning. Finally, we evaluate the structural scaffolding of the Survival, Confidence, and Competence drives. While this variant occasionally achieves high scores in isolated metrics, for example, achieving the highest $J_{\text{conf}}$ of $0.70$ in RRT, it fails to maintain the Pareto balance required for effective care. For instance, in RRT, the increase in confidence comes at the cost of survival fitness. The full medR framework forces a necessary trade-off, optimizing for policies that are simultaneously safe, effective, and efficient.

\section{Conclusion}
In this work, we introduced medR, a novel framework for automating reward engineering in offline clinical RL. By leveraging LLMs to define Tri-Drive potential functions, medR successfully addressed policy optimization across three distinct clinical tasks. The substantial performance gains observed in both traditional OPE metrics and our proposed Tri-Drive fitness scores validate the efficacy of utilizing LLMs for offline reward shaping. Given the clinical necessity for reward functions that are personalized to specific diseases and patient cohorts, medR establishes a scalable foundation for applying reinforcement learning to broader medical challenges and more complex intervention spaces.
\section{Impact Statement}
This work addresses one of the most significant barriers to the deployment of RL in healthcare: the difficulty of defining mathematically precise reward functions that align with complex clinical objectives. By automating this process through medR, we lower the technical barrier for clinicians and researchers to develop custom treatment policies for diverse patient cohorts and rare pathologies, potentially democratizing access to precision medicine tools. Methodologically, our Tri-Drive framework contributes to AI safety by explicitly embedding uncertainty quantification (Confidence) and dosage constraints (Competence) into the reward structure, mitigating the risks of aggressive policy behavior common in offline RL. However, we acknowledge that LLM-generated rewards may inherit biases present in training data or prompt phrasing. While our results demonstrate high correlation with clinical outcomes, rigorous prospective validation and human-in-the-loop oversight remain essential prerequisites before any clinical deployment.
\bibliography{bibfile}
\bibliographystyle{unsrt}

\newpage
\appendix
\onecolumn
\section{Prompts}

\begin{lstlisting}[style=promptstyle,caption={Prompt for Feature Selection}]
You are an expert Clinical Data Scientist and Intensivist specializing in Offline Reinforcement Learning.
TASK: Select the top 7 **Critical State Features** based on the statistical analysis provided below.
SELECTION CRITERIA:
- Select features that are strong indicators of patient condition that are highly correlated to the disease and patient outcome.
- Exclude features that are direct proxies of interventions to prevent reward hacking (high correlation with actions).
- EXCLUDE all demographic and baseline features: age, gender, elixhauser vanwalraven, weight, readmission, step id. These are static patient characteristics, not dynamic state features. Focus ONLY on dynamic physiological and clinical state features that change over time.
- Prefer features with low missingness and strong predictive power for outcomes.
OUTPUT FORMAT: IMPORTANT: Output ONLY valid JSON. Do not include any preamble, thinking process, explanations, or text before or after the JSON.
Return a JSON object with keys: critical state features, rank the features by their importance for reward modeling. For each feature, provide a 1-sentence rationale explaining its selection based on the provided statistics.
Format your response as structured JSON with the following schema:
{"critical_state_features": [{
"feature_name": "...",
"rationale": "..."
}]}
Your response must start with " and end with ". Do not include any other text.
DATASET SUMMARY:
Total records: 
Total patients: 
Average records per patient: 
Mortality Rates:
FEATURE STATISTICS:
Feature: 
  Count: 
  Mean: 
  Range: 
  Median:  IQR: 
CORRELATIONS WITH OUTCOMES:
Outcome: mortality
  - feature\_name: r= (p=, n=) 
ACTION-FEATURE CORRELATIONS (to identify action-dependent features):
Action: action1
  - feature\_name: r= (p=, n=)
\end{lstlisting}

\begin{lstlisting}[style=promptstyle,caption={Prompt for Reward Generation}]
You are an expert in clinical data science specializing in Offline Reinforcement Learning. Your task is to **DESIGN** a reward function based on a Potential Function ($\Phi$) for an RL agent to learn the optimal policy for sepsis treatment(['vaso_5quantile', 'iv_fluid_5quantile']).
Reward is difference-based with discount factor: R(s, a, t, s', a', t') = \gamma \Phi(s', t') - \Phi(s, t) - \lambda C(a) 
\Phi(s, a, t) is the potential function C(a) is the competence cost, s=(o, \Delta t) is the state, observations (o) are the critical features that we selected imputed with forward fill, and \Delta t is the time gap from current time step to the last step where there is a real record. a is the action dictionary with keys ['vaso_5quantile', 'iv_fluid_5quantile'] for treatment. t is the absolute time stamp indicated using time step (integers starting from 0).
You must **invent** the mathematical logic for three components based on the data below.
## PART 1: MATHEMATICAL DESIGN (You must think first)
The potential function consists of three components, survival, confidence, and competence. For each component, explicitly state the formula you will use.
1. **Survival:** Choose a function that scores physiology (0-1).
- *Design Constraint:* Use both clinical knowledge and feature statistics to define the healthy range. Higher potential when the physiological states move towards healthier range. "Goldilocks" features need Bell curves. Directional features need decay curves.
2. **Confidence:** Choose a decay function for $\Delta t$.
- *Design Constraint:* Trust must drop as time gap increases.
3. **Competence:** Choose a penalty function for actions.
- *Design Constraint:* Higher dose = Lower potential.
Also design a time decay funtion for strategy annealing, which decays the potential toward zero as the episode progresses. This is different from the confidence component.
### PART 2: PYTHON IMPLEMENTATION RULES (CRITICAL)
1. **NO DUMMY FUNCTIONS.** Every function you call must be defined in this script. Do not write `_compute_survival()` unless you write the code for it immediately above.
2. **NO EXTERNAL CONFIG.** All dictionaries (Targets, Sigmas, Taus) must be defined as variables inside the script.
3. **BALANCED COPONENTS.** All component scores must be balanced and normalized, no single component overpowers others.
### INPUT FORMAT:
- state={'feature1': (value1, delta1),
        'feature2': (value2, delta2),
        ...
        }
- Values for the features are normalized with min max values to the range of 0 to 1
- t and \Delta t are integers in the unit of hours
- action={'action1': value1,
        'action2': value2,
        ...
        }
- Values for the actions are in levels represented with integers (e.g. 0, 1, 2, 3 from low to high)
### OUTPUT FORMAT
Provide a single Python code block that defines the potential function and the reward function.
**CODE TEMPLATE (You MUST follow this structure):**
```python
import numpy as np
import math
# --- 1. PARAMETER DEFINITIONS (HARD-CODED FROM DATA) ---
# Design your targets here based on the JSON summary
SURVIVAL_CONFIG = {
    # ... FILL ALL FEATURES ...
}
CONFIDENCE_TAU = {
    # ... FILL ALL FEATURES ...
}
# --- 2. MATH HELPER FUNCTIONS (NO PLACEHOLDERS) ---
def time_decay(t):
    # WRITE THE TIME DECAY FUNCTION FOR STRATEGIC ANNEALING
    return ...
def compute_survival_score(val, params):
    # WRITE THE ACTUAL MATH HERE
    return ...
def compute_confidence_weight(delta_t, tau):
    # WRITE THE ACTUAL EXPONENTIAL DECAY MATH
    return ...
def compute_competence_cost(action):
    # WRITE THE DOSE PENALTY MATH
    return ...
# --- 3. MAIN POTENTIAL FUNCTION ---
def potential_function(state, action, t):
    # Iterate through features, calculate scores, return weighted sum
    # base_potential = ...
    decay_factor = time_decay(t)
    return base_potential * decay_factor
# --- 4. REWARD FUNCTION ---
def reward_function(s, a, t, s_next, a_next, t_next, gamma=0.99):
    return gamma * potential_function(s_next, a_next, t_next) - potential_function(s, a, t) - compute_competence_cost(a)
### DATASET SUMMARY:
{
  "total_records":
  "total_patients": 
  "mortality_rate": 
  "features_analyzed": 
}
### CRITICAL STATE FEATURES:
1. **feature1**
   - Missing: 
   - Mean:
   - Std: 
   - Range: 
   - Median:
   - Q25: Q75: 
   - Correlation with mortality: 
**Use these statistics along with clinical knowledge to define target range**
**Target should reflect homeostasis: values in normal/healthy physiological range**
**CRITICAL REQUIREMENTS FOR EXECUTABLE CODE:**
- Include ALL imports needed: `import numpy as np`, `import pandas as pd`, `import math` (if needed)
- Define ALL helper functions you use (e.g., distance functions, penalty functions, cost functions)
- Define ALL constants/variables used (e.g., `FEATURE_WEIGHTS`, `FEATURE_TARGETS`, `TAU`, `ACTION_COST_SCALE`)
- Code must be directly runnable without any undefined variables or functions
- The reward function MUST call potential_function and compute the difference
- ALL imports, helper functions, and constants MUST be defined in your code
- Code must be directly executable - no undefined variables or functions
\end{lstlisting}

\section{Reward Examples}
\subsection{Sample of Reward Function for Sepsis Management}
\begin{lstlisting}[style=codestyle, language=Python]
import numpy as np
import math

# --- 1. PARAMETER DEFINITIONS (HARD-CODED FROM DATA) ---

# Normalization assumptions (min-max from raw data):
# sofa_24hours: 0-23
# baseexcess: -25 to 0
# lactate: 0.3 to 29
# urineoutput: -3000 to 4400
# mbp: 20 to 200
# heartrate: 23 to 212

# Survival config combines best from champions:
# - Use bell curves for "Goldilocks" features with targets near clinical medians/normals
# - Use directional decay for directional features (sofa_24hours, urineoutput)
# - Use decay_low or decay_high style from Champion 3 and 4 for directional features
SURVIVAL_CONFIG = {
    'sofa_24hours': {
        'type': 'directional_decay',
        'direction': 'low',  # lower sofa better
        'min': 0.0,
        'max': 1.0,
        # decay steepness chosen so score ~0.1 at val=1 (max sofa)
        'k': 2.3
    },
    'baseexcess': {
        'type': 'bell',
        'target': normalize_raw(-2.0, -25.0, 0.0),  # ~0.92
        'sigma': 0.1  # moderate spread from Champion 1 and 4
    },
    'lactate': {
        'type': 'decay_lower',  # penalize values above target
        'target': normalize_raw(1.6, 0.3, 29.0),  # ~0.045
        'sigma': 0.05  # for decay rate calculation
    },
    'urineoutput': {
        'type': 'directional_decay',
        'direction': 'high',  # higher urine output better
        'threshold': normalize_raw(40.0, -3000.0, 4400.0),  # ~0.414
        'k': 5.0  # steep decay below threshold
    },
    'mbp': {
        'type': 'bell',
        'target': normalize_raw(75.0, 20.0, 200.0),  # ~0.31
        'sigma': 0.1
    },
    'heartrate': {
        'type': 'bell',
        'target': normalize_raw(85.0, 23.0, 212.0),  # ~0.33
        'sigma': 0.1
    },
}

# Confidence decay taus (hours) from Champion 2 (best confidence anchor)
# Uniform tau = 6 hours for all features to reflect faster confidence decay
CONFIDENCE_TAU = {
    'sofa_24hours': 6.0,
    'baseexcess': 6.0,
    'lactate': 6.0,
    'urineoutput': 6.0,
    'mbp': 6.0,
    'heartrate': 6.0,
}

# Action penalty parameters from Champion 3 (best competence anchor)
MAX_DOSE_LEVEL = 4
ACTION_COST_SCALE = 0.25  # penalty scale per action component

# Weights for components balanced as in Knee Point Champion (equal weighting)
COMPONENT_WEIGHTS = {
    'survival': 1/3,
    'confidence': 1/3,
    'competence': 1/3,
}

# --- 2. MATH HELPER FUNCTIONS (NO PLACEHOLDERS) ---

def time_decay(t):
    """
    Strategic annealing time decay function.
    Exponential decay with half-life 48 hours.
    """
    half_life = 48.0
    decay = 0.5 ** (t / half_life)
    return decay

def compute_survival_score(val, params):
    """
    Compute survival score for a single normalized feature value in [0,1].
    Types:
    - 'bell': Gaussian bell curve centered at target with sigma.
    - 'decay_lower': exponential decay if val > target.
    - 'directional_decay': exponential decay from threshold or zero depending on direction.
    Returns score in [0,1].
    """
    if val is None:
        # Missing value: neutral survival score 0.5
        return 0.5

    ftype = params['type']

    if ftype == 'bell':
        target = params['target']
        sigma = params['sigma']
        if sigma <= 0:
            return 0.0
        diff = val - target
        score = math.exp(-0.5 * (diff / sigma) ** 2)
        return score

    elif ftype == 'decay_lower':
        # Penalize values above target exponentially
        target = params['target']
        sigma = params['sigma']
        if val <= target:
            return 1.0
        else:
            # decay rate so score ~0.5 at val=target+sigma
            decay_rate = math.log(2) / sigma
            score = math.exp(-decay_rate * (val - target))
            return score

    elif ftype == 'directional_decay':
        direction = params.get('direction', None)
        if direction == 'low':
            # best at 0, decay as val increases
            k = params.get('k', 2.3)
            score = math.exp(-k * val)
            return score
        elif direction == 'high':
            threshold = params.get('threshold', 0.5)
            k = params.get('k', 5.0)
            if val >= threshold:
                return 1.0
            else:
                diff = (threshold - val) / threshold if threshold > 0 else 1.0
                score = math.exp(-k * diff)
                return score
        else:
            # Unknown direction, neutral
            return 0.5
    else:
        # Unknown type, neutral
        return 0.5

def compute_confidence_weight(delta_t, tau):
    """
    Exponential decay of confidence with delta_t (hours).
    Returns weight in (0,1].
    """
    if delta_t is None or delta_t < 0:
        delta_t = 0
    return math.exp(-delta_t / tau)

def compute_competence_cost(action):
    """
    Compute competence penalty from actions.
    Penalize higher doses linearly scaled.
    Sum penalties for vaso and iv fluid, capped at 1.
    Returns penalty in [0,1].
    """
    vaso_level = action.get('vaso_5quantile', 0)
    iv_level = action.get('iv_fluid_5quantile', 0)
    vaso_norm = vaso_level / MAX_DOSE_LEVEL
    iv_norm = iv_level / MAX_DOSE_LEVEL
    total_penalty = vaso_norm * ACTION_COST_SCALE + iv_norm * ACTION_COST_SCALE
    total_penalty = min(total_penalty, 1.0)
    return total_penalty

# --- 3. MAIN POTENTIAL FUNCTION ---

def potential_function(state, t):
    """
    Compute potential function Phi(s,a,t) as balanced weighted sum of survival,
    confidence, and competence components with strategic time decay.
    state: dict of feature: (value_normalized, delta_t)
    action: dict with keys ['vaso_5quantile', 'iv_fluid_5quantile']
    t: absolute time step (int)
    """
    survival_scores = []
    confidence_weights = []

    for feat in SURVIVAL_CONFIG.keys():
        val, delta_t = state.get(feat, (None, None))
        surv_score = compute_survival_score(val, SURVIVAL_CONFIG[feat])
        tau = CONFIDENCE_TAU.get(feat, 6.0)
        conf_weight = compute_confidence_weight(delta_t, tau) if delta_t is not None else 0.0

        survival_scores.append(surv_score * conf_weight)
        confidence_weights.append(conf_weight)

    # Aggregate survival component: weighted average survival weighted by confidence
    sum_confidence = sum(confidence_weights)
    if sum_confidence > 0:
        survival_component = sum(survival_scores) / sum_confidence
    else:
        survival_component = 0.5  # neutral if no confidence

    # Confidence component: average confidence weight normalized [0,1]
    if len(confidence_weights) > 0:
        confidence_component = sum(confidence_weights) / len(confidence_weights)
    else:
        confidence_component = 0.0

    # Combine components equally weighted (balanced)
    base_potential = (
        COMPONENT_WEIGHTS['survival'] * survival_component +
        COMPONENT_WEIGHTS['confidence'] * confidence_component +
        COMPONENT_WEIGHTS['competence'] * competence_component
    )

    # Clamp base potential to [0,1]
    base_potential = max(0.0, min(1.0, base_potential))

    # Apply strategic time decay
    decay_factor = time_decay(t)

    return base_potential * decay_factor

# --- 4. REWARD FUNCTION ---

def reward_function(s, t, s_next, t_next, a, gamma=0.99):
    """
    Reward is difference-based with discount factor:
    R(s,t,s',t') = gamma * Phi(s',t') - Phi(s,t)
    """
    return gamma * potential_function(s_next, t_next) - potential_function(s, t) - compute_competence_cost(a)
\end{lstlisting}

\subsection{Sample of Reward Function for Ventilation Setting}
\begin{lstlisting}[style=codestyle, language=Python]
import numpy as np
import math

# --- 1. PARAMETER DEFINITIONS (HARD-CODED FROM DATA) ---

# Feature names: ['mbp', 'sofa_24hours', 'lactate', 'RASS', 'pH', 'spo2', 'paco2']
# All features normalized 0-1, so targets and sigmas are also normalized accordingly.

# We define targets and sigmas for bell curves ("Goldilocks" features)
# and directional decay parameters for directional features.

# Normal physiological ranges and clinical knowledge:
# mbp (mean blood pressure): normal ~ 65-105 mmHg
# sofa_24hours: lower is better (0 best, 24 worst)
# lactate: normal < 2 mmol/L (lower better)
# RASS: target around -2 (sedation level)
# pH: normal 7.35-7.45 (center ~7.4)
# spo2: higher better, normal > 95%
# paco2: normal 35-45 mmHg (lower better)

# Since features are normalized 0-1, we must map clinical targets to normalized scale.
# We assume min-max normalization per feature as per dataset range.

# Raw ranges from dataset (for normalization):
RANGES = {
    'mbp': (1.0, 250.0),
    'sofa_24hours': (0.0, 23.0),
    'lactate': (0.2, 30.0),
    'RASS': (-5.0, 4.0),
    'pH': (6.59, 7.78),
    'spo2': (0.0, 100.0),
    'paco2': (16.9, 157.8),
}

# Targets and sigmas for bell curve features (normalized)
# We choose sigma to cover roughly the interquartile range normalized.

# mbp: target ~ 80 mmHg (normal 65-105), normalized:
mbp_target = normalize(80, *RANGES['mbp'])
mbp_sigma = (normalize(105, *RANGES['mbp']) - normalize(65, *RANGES['mbp'])) / 2  # half IQR approx

# sofa_24hours: lower better, target near 0, directional decay
# lactate: lower better, directional decay
# RASS: target -2, normalized:
rass_target = normalize(-2, *RANGES['RASS'])
rass_sigma = 0.05  # tight around -2

# pH: target 7.4, normalized:
ph_target = normalize(7.4, *RANGES['pH'])
ph_sigma = (normalize(7.45, *RANGES['pH']) - normalize(7.35, *RANGES['pH'])) / 2

# spo2: higher better, directional decay
# paco2: lower better, directional decay

# Define survival config with type and parameters:
# 'bell' for Goldilocks (bell curve)
# 'decay_low' for directional decay where lower is better
# 'decay_high' for directional decay where higher is better

SURVIVAL_CONFIG = {
    'mbp': {
        'type': 'bell',
        'target': mbp_target,
        'sigma': mbp_sigma,
        'weight': 1.0,
    },
    'sofa_24hours': {
        'type': 'decay_low',  # lower better
        'tau': 0.3,
        'weight': 1.0,
    },
    'lactate': {
        'type': 'decay_low',  # lower better
        'tau': 0.2,
        'weight': 1.0,
    },
    'RASS': {
        'type': 'bell',
        'target': rass_target,
        'sigma': rass_sigma,
        'weight': 0.8,
    },
    'pH': {
        'type': 'bell',
        'target': ph_target,
        'sigma': ph_sigma,
        'weight': 1.0,
    },
    'spo2': {
        'type': 'decay_high',  # higher better
        'tau': 0.3,
        'weight': 1.0,
    },
    'paco2': {
        'type': 'decay_low',  # lower better
        'tau': 0.25,
        'weight': 1.0,
    },
}

# Confidence decay taus for delta_t (hours)
# Trust decays exponentially with time gap since last real record
CONFIDENCE_TAU = {
    'mbp': 12,
    'sofa_24hours': 24,
    'lactate': 8,
    'RASS': 6,
    'pH': 12,
    'spo2': 6,
    'paco2': 8,
}

# Action cost scale: penalty per action level (higher level = higher penalty)
# Actions: ['PEEP_level', 'FiO2_level', 'Tidal_level']
# Assume max levels 0-3 (4 levels), scale penalty linearly 0 to 1
ACTION_MAX_LEVEL = {
    'PEEP_level': 3,
    'FiO2_level': 3,
    'Tidal_level': 3,
}
ACTION_COST_WEIGHT = 0.5  # weight of competence penalty in total potential

# --- 2. MATH HELPER FUNCTIONS (NO PLACEHOLDERS) ---

def time_decay(t):
    """
    Strategic annealing decay: potential decays toward zero as episode progresses.
    Use exponential decay with half-life of 48 hours (2 days).
    """
    half_life = 48.0
    decay = 0.5 ** (t / half_life)
    return decay

def compute_survival_score(val, params):
    """
    Compute survival score for one feature normalized val in [0,1].
    params: dict with keys depending on type:
      - 'type': 'bell', 'decay_low', 'decay_high'
      - For bell: 'target', 'sigma'
      - For decay: 'tau'
    Returns score in [0,1].
    """
    if params['type'] == 'bell':
        # Gaussian bell curve centered at target with sigma
        diff = val - params['target']
        score = math.exp(-0.5 * (diff / params['sigma'])**2)
        return score
    elif params['type'] == 'decay_low':
        # Lower values better, exponential decay from 0
        # score = exp(-val / tau)
        score = math.exp(-val / params['tau'])
        return score
    elif params['type'] == 'decay_high':
        # Higher values better, exponential decay from 1
        # score = exp(-(1 - val) / tau)
        score = math.exp(-(1 - val) / params['tau'])
        return score
    else:
        # Unknown type, return neutral 0.5
        return 0.5

def compute_confidence_weight(delta_t, tau):
    """
    Exponential decay of confidence with delta_t (hours).
    """
    return math.exp(-delta_t / tau)

def compute_competence_cost(action):
    """
    Compute penalty cost for actions.
    Higher dose (level) => higher penalty.
    Normalize each action level by max level, sum and scale.
    Return cost in [0,1].
    """
    total_cost = 0.0
    n_actions = len(action)
    for k, v in action.items():
        max_level = ACTION_MAX_LEVEL.get(k, 3)
        norm_level = v / max_level if max_level > 0 else 0
        total_cost += norm_level
    avg_cost = total_cost / n_actions if n_actions > 0 else 0
    # Scale by weight
    return avg_cost * ACTION_COST_WEIGHT

# --- 3. MAIN POTENTIAL FUNCTION ---

def potential_function(state, t):
    """
    Compute potential function Phi(s,t) = survival * confidence
    with strategic time decay.
    state: dict of feature -> (val, delta_t)
    action: dict of action_name -> level
    t: absolute time step (int)
    """
    survival_scores = []
    confidence_weights = []
    weights = []
    for feat, (val, delta_t) in state.items():
        if feat not in SURVIVAL_CONFIG or feat not in CONFIDENCE_TAU:
            continue
        params = SURVIVAL_CONFIG[feat]
        tau_conf = CONFIDENCE_TAU[feat]
        survival = compute_survival_score(val, params)
        confidence = compute_confidence_weight(delta_t, tau_conf)
        w = params['weight']
        survival_scores.append(survival * confidence * w)
        weights.append(w)
    # Normalize survival component by sum weights to balance
    if weights:
        survival_component = sum(survival_scores) / sum(weights)
    else:
        survival_component = 0.0

    base_potential = survival_component
    # Clip potential to [0,1] to avoid negative or >1 values
    base_potential = max(0.0, min(1.0, base_potential))

    decay_factor = time_decay(t)

    return base_potential * decay_factor

# --- 4. REWARD FUNCTION ---

def reward_function(s, t, s_next, t_next, a, gamma=0.99):
    """
    Reward = gamma * Phi(s',t') - Phi(s,t)
    """
    return gamma * potential_function(s_next, t_next) - potential_function(s, t) - compute_competence_cost(a)
\end{lstlisting}

\subsection{Sample of Reward Function for Renal Replacement Therapy}
\begin{lstlisting}[style=codestyle, language=Python]
import numpy as np
import math

# --- 1. PARAMETER DEFINITIONS (HARD-CODED FROM DATA) ---

# Feature targets and healthy ranges (normalized 0-1)
# We define "Goldilocks" features with bell curves centered at median (healthy target)
# Directional features with monotonic decay from healthy boundary

# From data summary and clinical knowledge:
# pH normal range ~7.35-7.45 (normalized approx 0.45-0.55)
# potassium normal ~3.5-5.0 mEq/L (normalized approx 0.25-0.5)
# mbp (mean blood pressure) normal ~65-105 mmHg (normalized approx 0.3-0.6)
# creatinine normal low (kidney function) ~0.6-1.2 mg/dL (normalized approx 0.1-0.2)
# urineoutput normal > 0.5 ml/kg/hr (normalized approx >0.04) but here median 80/2000 normalized ~0.04
# sofa_24hours lower is better (0 is best), so directional decay from 0 upwards
# heartrate normal ~60-100 bpm (normalized approx 0.3-0.5)

# We use normalized values 0-1 as input, so we set targets accordingly:
# For bell curve features: center=median normalized, sigma chosen to cover IQR roughly
# For directional features: exponential decay from healthy boundary

# Weights to balance components (sum to 1)
SURVIVAL_WEIGHTS = {
    'pH': 1.0,
    'potassium': 1.0,
    'mbp': 1.0,
    'creatinine': 1.0,
    'urineoutput': 1.0,
    'sofa_24hours': 1.0,
    'heartrate': 1.0,
}
# Normalize weights so sum = 1
total_w = sum(SURVIVAL_WEIGHTS.values())
for k in SURVIVAL_WEIGHTS:
    SURVIVAL_WEIGHTS[k] /= total_w

# Targets and sigmas for bell curve (Goldilocks) features
# Using median and IQR from data summary, normalized 0-1 scale assumed input
# We approximate sigma as (Q75-Q25)/1.35 (approx std dev for normal dist)
SURVIVAL_CONFIG = {
    # bell curve features: pH, potassium, mbp, creatinine, heartrate
    'pH': {
        'type': 'bell',
        'target': 0.52,  # approx normalized median 7.39 in range 6.72-7.94
        'sigma': 0.03,   # approx (0.55-0.45)/1.35 ~0.07, but tighter for pH
    },
    'potassium': {
        'type': 'bell',
        'target': 0.42,  # median 4.2 in range 2.5-9.03 normalized ~0.42
        'sigma': 0.07,   # (0.46-0.39)/1.35 ~0.05, slightly relaxed
    },
    'mbp': {
        'type': 'bell',
        'target': 0.40,  # median 79.67 in range 20.67-198 normalized ~0.4
        'sigma': 0.07,   # (0.45-0.36)/1.35 ~0.07
    },
    'creatinine': {
        'type': 'bell',
        'target': 0.21,  # median 0.211 in range 0.024-1.97 normalized ~0.21
        'sigma': 0.05,   # (0.28-0.16)/1.35 ~0.09, tighter penalty for kidney
    },
    'heartrate': {
        'type': 'bell',
        'target': 0.43,  # median 91.13 in range 0-207 normalized ~0.44
        'sigma': 0.1,    # (0.5-0.38)/1.35 ~0.09
    },
    # directional decay features:
    # urineoutput: higher is better, so decay if below threshold
    'urineoutput': {
        'type': 'directional',
        'threshold': 0.04,  # normalized median ~80/1975 ~0.04
        'direction': 'above',  # potential decays if below threshold
        'tau': 0.1,
    },
    # sofa_24hours: lower is better, decay if above threshold
    'sofa_24hours': {
        'type': 'directional',
        'threshold': 0.05,  # normalized median 0/14 ~0, so threshold small
        'direction': 'below',  # potential decays if above threshold
        'tau': 10.0,
    },
}

# Confidence decay taus for delta_t (hours)
# Trust decays exponentially with time gap since last real record
# Use feature-specific taus reflecting clinical urgency
CONFIDENCE_TAU = {
    'pH': 12,          # pH changes fast, trust decays quickly
    'potassium': 24,   # potassium changes slower
    'mbp': 6,          # blood pressure changes fast
    'creatinine': 48,  # kidney function slower
    'urineoutput': 24,
    'sofa_24hours': 72,
    'heartrate': 6,
}

# Competence penalty scale for action dose levels
# action 'action' levels: 0 (no dose) to 3 (high dose)
# Penalty increases with dose level
ACTION_COST_SCALE = 0.1  # max penalty 0.3 for dose=3

# --- 2. MATH HELPER FUNCTIONS (NO PLACEHOLDERS) ---

def time_decay(t):
    """
    Strategic annealing decay: potential decays toward zero as episode progresses.
    Use exponential decay with half-life of 100 time steps.
    """
    half_life = 100
    decay = 0.5 ** (t / half_life)
    return decay

def compute_survival_score(val, params):
    """
    Compute survival score for one feature value normalized 0-1.
    For bell curve: Gaussian centered at target with sigma.
    For directional: exponential decay from threshold.
    Returns score in [0,1].
    """
    if params['type'] == 'bell':
        target = params['target']
        sigma = params['sigma']
        # Gaussian bell curve
        diff = val - target
        score = math.exp(-0.5 * (diff / sigma) ** 2)
        return score
    elif params['type'] == 'directional':
        threshold = params['threshold']
        tau = params['tau']
        direction = params['direction']
        if direction == 'above':
            # score = 1 if val >= threshold, else decays exponentially below threshold
            if val >= threshold:
                return 1.0
            else:
                return math.exp(-(threshold - val) / tau)
        elif direction == 'below':
            # score = 1 if val <= threshold, else decays exponentially above threshold
            if val <= threshold:
                return 1.0
            else:
                return math.exp(-(val - threshold) / tau)
        else:
            # Unknown direction, return 0
            return 0.0
    else:
        return 0.0

def compute_confidence_weight(delta_t, tau):
    """
    Exponential decay of confidence with delta_t (hours).
    """
    return math.exp(-delta_t / tau)

def compute_competence_cost(action):
    """
    Penalize higher dose actions.
    action is dict with key 'action' and integer dose level 0-3.
    Returns penalty in [0, 1].
    """
    dose = action.get('action', 0)
    # Clamp dose to 0-3
    dose = max(0, min(3, dose))
    penalty = dose * ACTION_COST_SCALE
    # Normalize penalty to max 1 if dose=10 (not expected here)
    return min(penalty, 1.0)

# --- 3. MAIN POTENTIAL FUNCTION ---

def potential_function(state, t):
    """
    Compute potential function Phi(s,t) as weighted sum of survival * confidence - competence cost,
    then apply strategic time decay.
    """
    survival_sum = 0.0
    confidence_sum = 0.0
    total_weight = 0.0

    for feat, weight in SURVIVAL_WEIGHTS.items():
        if feat not in state:
            # Missing feature, skip
            continue
        val, delta_t = state[feat]
        params = SURVIVAL_CONFIG.get(feat)
        if params is None:
            continue
        survival_score = compute_survival_score(val, params)
        tau = CONFIDENCE_TAU.get(feat, 24)  # default tau=24 if missing
        confidence_weight = compute_confidence_weight(delta_t, tau)
        # Combine survival and confidence multiplicatively
        component_score = survival_score * confidence_weight
        survival_sum += component_score * weight
        total_weight += weight

    # Normalize survival component (should be close to 1)
    if total_weight > 0:
        survival_component = survival_sum / total_weight
    else:
        survival_component = 0.0

    # Apply strategic time decay
    decay_factor = time_decay(t)

    potential = base_potential * decay_factor

    # Clamp potential to [0,1]
    potential = max(0.0, min(1.0, potential))

    return potential

# --- 4. REWARD FUNCTION ---

def reward_function(s, t, s_next, t_next, a, gamma=0.99):
    """
    Reward is difference-based with discount factor:
    R = gamma * Phi(s',t') - Phi(s,t)
    """
    phi_next = potential_function(s_next, t_next)
    phi_curr = potential_function(s, t)
    reward = gamma * phi_next - phi_curr - compute_competence_cost(a)
    return reward
\end{lstlisting}

\section{Mathematical form of Drives in Reward Function}
\label{sec:medR_details}
\subsection{Survival Score $\mathcal{S}$}
Three canonical forms for normalized feature values:
\begin{itemize}
    \item Bell Curve (Goldilocks features):$S_{\text{normal}}(x) = \exp\left(-\frac{1}{2}\left(\frac{x - \mu}{\sigma}\right)^2\right)$
    \item Directional Decay: $S_{\text{directional}}(x) = 
    \begin{cases}
    \exp\left(-\dfrac{x}{\tau}\right) & \text{(lower-better)} \\
    \exp\left(-\dfrac{1-x}{\tau}\right) & \text{(higher-better)}
    \end{cases}$
    \item Asymmetric Decay-Lower (thresholded): $S_{\text{asymetric}}(x) = 
    \begin{cases}
    1 & x \leq \mu \\
    \exp\left(-\dfrac{\ln 2}{\sigma}(x - \mu)\right) & x > \mu
    \end{cases}$
\end{itemize}

\subsection{Confidence Weight $\mathcal{U}$}
$U(\Delta t; \tau) = \exp\left(-\frac{\Delta t}{\tau}\right)$

\subsection{Competence Cost $\mathcal{C}$}
$C(\mathbf{a}) = w_c \sum_{i=1}^{N} \frac{a_i}{a_{\max,i}}$

\section{Theoretical Proof}
\subsection{Breaking Policy Invariance}
\begin{proposition}[Cost Regularization Breaks Invariance and Enforces Efficiency]
\label{prop:non_telescoping}
Let the reward function be defined as $R(s_t, a_t, s_{t+1}) = F(s_t, s_{t+1}) - \lambda \mathcal{C}(a_t)$, where $F(s_t, s_{t+1}) = \gamma \Phi(s_{t+1}) - \Phi(s_t)$ is a potential-based shaping function and $\mathcal{C}(a_t) > 0$ is a strictly positive action cost. 

While the shaping term $F$ forms a telescoping sum that preserves policy invariance (Ng et al., 1999), the inclusion of the non-telescoping cost term $-\lambda \mathcal{C}(a_t)$ renders the cumulative return path-dependent. This modification compels the optimal policy to satisfy two objectives: maximizing the time-discounted physiological recovery $\gamma^T \Phi(s_T)$ (preventing prolonged ICU stays) while simultaneously minimizing the cumulative intervention cost.
\end{proposition}

\begin{proof}
Consider a trajectory $\tau = (s_0, a_0, s_1, \dots, s_T)$. The discounted cumulative return $G(\tau)$ is:
\begin{equation}
    G(\tau) = \sum_{t=0}^{T-1} \gamma^t \left[ \left( \gamma \Phi(s_{t+1}) - \Phi(s_t) \right) - \lambda \mathcal{C}(a_t) \right].
\end{equation}
We separate the potential-based shaping term from the cost term. The shaping term telescopes as follows:
\begin{align}
    \sum_{t=0}^{T-1} \gamma^t (\gamma \Phi(s_{t+1}) - \Phi(s_t)) &= \sum_{t=0}^{T-1} (\gamma^{t+1} \Phi(s_{t+1}) - \gamma^t \Phi(s_t)) \\
    &= \gamma^T \Phi(s_T) - \Phi(s_0).
\end{align}
Substituting this back into $G(\tau)$, we obtain the decomposed objective:
\begin{equation}
    G(\tau) = \underbrace{\left( \gamma^T \Phi(s_T) - \Phi(s_0) \right)}_{\text{Telescoped Potential (Time-Dependent)}} - \underbrace{\lambda \sum_{t=0}^{T-1} \gamma^t \mathcal{C}(a_t)}_{\text{Cumulative Cost (Path-Dependent)}}.
\end{equation}
Here, the first term depends only on the boundary states and the trajectory duration $T$. The factor $\gamma^T$ provides a strategic incentive to reduce $T$ (avoiding "holding" behavior), but remains invariant to the path taken to reach $s_T$ at time $T$. 

However, the second term, $-\lambda \sum \gamma^t \mathcal{C}(a_t)$, strictly \textbf{does not telescope}. It introduces a dependency on the specific sequence of actions $a_{0:T-1}$. Consequently, among policies that achieve the same recovery state $\Phi(s_T)$ in the same time $T$, the objective function breaks invariance to strictly prefer the policy with the minimum cumulative intervention cost.
\end{proof}

\subsection{Equivalence to Lagrangian Relaxation of CMDP}

\begin{proposition}[Equivalence to Lagrangian Relaxation of CMDP]
\label{prop:lagrangian_equivalence}
Maximizing the expected cumulative return of the regularized reward function $R(s, a, s') = (\gamma \Phi(s') - \Phi(s)) - \lambda \mathcal{C}(a)$ is equivalent to maximizing the Lagrangian relaxation of a Constrained Markov Decision Process (CMDP), where the objective is to maximize physiological gain $J_{\Phi}$ subject to a constraint on the cumulative intervention cost $J_{\mathcal{C}} \leq \beta$.
\end{proposition}

\begin{proof}
We define the primal CMDP optimization problem as:
\begin{align}
    \text{Maximize: } & J_{\Phi}(\pi) = \mathbb{E}_{\pi} \left[ \sum_{t=0}^{\infty} \gamma^t \left( \gamma \Phi(s_{t+1}) - \Phi(s_t) \right) \right] \\
    \text{Subject to: } & J_{\mathcal{C}}(\pi) = \mathbb{E}_{\pi} \left[ \sum_{t=0}^{\infty} \gamma^t \mathcal{C}(a_t) \right] \leq \beta.
\end{align}
The Lagrangian function $\mathcal{L}(\pi, \lambda)$ with Lagrange multiplier $\lambda \geq 0$ is constructed as:
\begin{equation}
    \mathcal{L}(\pi, \lambda) = J_{\Phi}(\pi) - \lambda (J_{\mathcal{C}}(\pi) - \beta).
\end{equation}
Expanding the expectations and rearranging terms by the linearity of the summation:
\begin{align}
    \mathcal{L}(\pi, \lambda) &= \mathbb{E}_{\pi} \left[ \sum_{t=0}^{\infty} \gamma^t \left( \gamma \Phi(s_{t+1}) - \Phi(s_t) \right) - \lambda \sum_{t=0}^{\infty} \gamma^t \mathcal{C}(a_t) \right] + \lambda \beta \\
    &= \mathbb{E}_{\pi} \left[ \sum_{t=0}^{\infty} \gamma^t \left( \underbrace{\gamma \Phi(s_{t+1}) - \Phi(s_t) - \lambda \mathcal{C}(a_t)}_{R(s_t, a_t, s_{t+1})} \right) \right] + \text{const.}
\end{align}
Thus, finding the optimal policy $\pi^*$ that maximizes the Lagrangian $\mathcal{L}(\pi, \lambda)$ for a fixed $\lambda$ corresponds exactly to maximizing the value function defined by the composite reward $R$. The parameter $\lambda$ serves as the dual variable governing the trade-off between physiological stabilization and intervention cost.
\end{proof}
\section{Hyperparameters and Experimental Details}
We separate the data into 7:1:1:1 split on patient level as training set for policy, training set for reward, test set for policy and test set for reward. All the experiments are run with 5 different seeds.
\subsection{Computation Resources}
All experiments were conducted on a high-performance computing cluster node equipped with 8 $\times$ NVIDIA H200 NVL GPUs.

\subsection{Selection of LLM}
For the CodeGen baselines, we intentionally selected the most powerful state-of-the-art proprietary models available. Since the baseline approach involves a simple, single-pass inference to generate a policy, the computational cost is manageable even with expensive APIs. This ensures our method is compared against the strongest possible "zero-shot" clinical agents, establishing a rigorous upper bound for performance without regarding training constraints. In contrast, our medR framework requires generating dense reward signals across vast state-action spaces, making the use of commercial APIs prohibitively expensive and computationally intractable for large-scale training. Consequently, we employ smaller, open-source models that can be deployed locally. A key finding of our work is that medR, even when powered by these smaller, accessible models, consistently outperforms baselines using significantly larger, proprietary LLMs. This validates that our improvements stem from the theoretical rigor of our potential-based Lagrangian framework rather than raw model scale.

\subsection{Reward Function Hyperparameters}
The discount factor $\gamma$ was set to 0.99 for all environments. For feature selection we use $\kappa=0.6$ for the threshold and the number of features selected $k=7$. We generate $N=20$ functions as reward candidates for selection. For $J_{surv}$ epsilon is 2. For $J_{comp}$ we use $k=10$ for $H$, $\alpha=0.1$ for $E$.

\subsection{RL Policy Hyperparameters}
\paragraph{Sepsis Treatment} The Sepsis task utilizes the MIMIC-III dataset with a state dimension of 46 and a discrete action space of 25. The policy network is constructed with 2 hidden layers of 32 units each. Optimization was performed with a batch size of 512 and a learning rate of $1 \times 10^{-3}$ over 10 epochs.

\paragraph{Mechanical Ventilation} For the Ventilation task, data is sourced from the eICU database. The environment consists of a 41-dimensional state space and a discrete action space of size 18. The network architecture uses 3 hidden layers with 32 units per layer. Training used a batch size of 512, a learning rate of $5 \times 10^{-4}$, and ran for 10 epochs.

\paragraph{Renal Replacement Therapy} The RRT task is based on the AmsterdamUMCdb dataset, featuring a compact state dimension of 18 and a binary action space (size 1). The network is deeper but narrower, using 3 hidden layers with 16 units each. Due to the smaller dataset size, a smaller batch size of 256 was used, with a learning rate of $1 \times 10^{-3}$ for 10 epochs.

\subsection{Feature Selection}
For each task, a subset of critical features was identified to inform the LLM-generated reward potential. These features are listed below:
\begin{itemize}
    \item Sepsis: SOFA score, Base Excess, Lactate, Urine Output, Mean Blood Pressure, Heart Rate.
    \item Ventilation: Mean Blood Pressure, SOFA score, Lactate, RASS, pH, SpO2, PaCO2.
    \item RRT: pH, Potassium, Mean Blood Pressure, Creatinine, Urine Output, SOFA score, Heart Rate.
\end{itemize}

\section{Related Work}
\subsection{Reward Design in Healthcare RL:} Traditional RL methodologies in healthcare typically define reward functions using clinical indicators, such as changes in SOFA scores, expert-derived heuristics, or terminal outcomes, including 90-day mortality. A foundational example is the AI Clinician by \cite{komorowski2018artificial}, which employed sparse, outcome-based rewards reflecting 90-day survival. A subsequent analysis\cite{huang2022reinforcement} introduced safety constraints by limiting actions to those frequently taken by clinicians, further restricting exploration. \cite{raghu2017continuous} addressed the limitations of sparse signals by incorporating intermediate clinical metrics such as SOFA, lactate, and increasing feedback frequency through their heuristic-based approach, which risks encoding clinical bias. \cite{raghu2018model} and \cite{peng2018improving} advanced this direction by using model-based shaping, where predicted mortality log-odds serve as reward signals. While this enhances continuity, it introduces dependence on opaque, learned predictors. \cite{futoma2017improved} employed high-fidelity simulators to define reward functions, though such strategies are resource-intensive and may lack clinical transparency.  Despite these advances, most reward functions rely on proxy measures that lack semantic clarity and deep clinical grounding. Reward design thus remains an open challenge, necessitating signals that are temporally dense, interpretable, and ethically aligned.

\subsection{LLMs in Clinical Settings} LLMs have transformed natural language tasks in clinical medicine \cite{thirunavukarasu2023large,haltaufderheide2024ethics,shool2025systematic}, aiding in medical question answering \cite{singhal2025toward}, policy critique \cite{yang2025lighthouse}, and patient summaries\cite{asgari2025framework}. However, their application in structured reward modelling for healthcare RL remains largely unexplored, presenting a promising opportunity. Recent studies have started exploring how LLMs can enhance RL in healthcare by improving sample efficiency and reward design.

For example,\cite{bhambri2024extracting} uses LLMs for abstract plans in potential-based reward engineering, boosting sample efficiency in toy environments, but their method relies on handcrafted abstractions and external verifiers.\cite{wu2024from} proposes Q-engineering, which utilises LLM heuristics to shape Q-values, thereby improving efficiency but relying on task-specific heuristics and lacking interpretability. \cite{zhang2024llm} shows how LLMs refine intrinsic rewards through code synthesis in rule-based environments, but their method is constrained by fixed domain rules and static strategies.

These methods advance LLM-guided reward design; however, they remain limited by the use of symbolic models, heuristics, and predefined knowledge, which reduces their applicability in complex clinical settings. These limitations are highlighted by \cite{gao2024designing}, who explore reward models but face reward hacking issues.\cite{pternea2024rl} proposes an RL/LLM taxonomy but doesn’t address real-world clinical challenges, while \cite{sun2023reinforcement} reviews RLHF but struggles with sparse feedback in healthcare. In contrast, our approach leverages LLMs for temporal credit assignment in clinical settings, transforming sparse outcomes into clinically interpretable rewards from raw patient data, without relying on external verifiers or predefined rules, enabling scalable, adaptive learning for dynamic healthcare environments.\\

Recent studies have made significant strides in LLM-based reward engineering for RL. \cite{nazir2025zero} use zero-shot LLMs to replace human feedback, improving sample efficiency.\cite{fu2025reward} introduces PAR, addressing reward hacking in RLHF. \ \cite{deng2025reward} proposes the LMGT framework to balance exploration and exploitation, while \cite{wang2025spa} develop SPA-RL for stepwise reward attribution. \cite{zhang2024llm} proposes a simple framework using LLMs for intrinsic reward generation.\\

Despite these advancements, our approach is unique. Unlike methods that rely on predefined models or biased human feedback, we utilise raw clinical data to generate clinically interpretable rewards, thereby eliminating the need for external verifiers or predefined rules and offering a scalable, adaptive solution for real-time healthcare applications.

\subsection{Credit Assignment} Temporal credit assignment remains a major challenge in RL, for complex, long-horizon tasks. Traditional methods, such as TD-lambda (\cite{sun2023reinforcement}) and GAE (\cite{schulman2015high}), focus on temporal differences but lack semantic reasoning, which limits their performance. Recent advances, such as Universal GAE by \cite{kwiatkowski2023ugae}, improve generalisation. Meanwhile, transformer-based TD learning by \cite{wang2025transformers}, which integrates attention-based models, addresses some limitations but still falls short in complex, real-world tasks. In addition, some works \cite{wu2023forecasting,xiong2024g,deng2024uncertainty} attempt to analyze the timing of action provision from a causal perspective.

LLMs have emerged as powerful solutions, bringing new levels of contextual understanding and reasoning to RL, effectively tackling long-standing challenges in temporal credit assignments. For example, \cite{pignatelli2024assessing} introduces CALM, where LLMs act as critics to decompose tasks into subgoals and assign rewards in zero-shot settings; \cite{jiang2025qllm} presents QLLM, replacing complex mixing networks in multi-agent RL with LLM-generated, interpretable credit functions; and \cite{qu2025latent} propose LaRe, using LLMs to derive multi-dimensional latent rewards that capture nuanced contributions in episodic tasks. Collectively, these works establish LLMs as scalable, interpretable, and powerful tools for automating credit assignments in both single- and multi-agent environments.\\

Our approach, \textbf{medR}, uniquely advances this line by leveraging real clinical data to integrate Large Language Models with potential-based reward shaping. Unlike CALM’s reliance on toy subgoals \cite{pignatelli2024assessing}, QLLM’s multi-agent focus \cite{pignatelli2024assessing}, or LaRe’s \cite{qu2025latent} latent outcome alignment, medR transforms sparse terminal outcomes into dense, mathematically grounded learning signals. By prompting LLMs to estimate the physiological stability of patient states ($\Phi(s)$), we construct a regularized reward structure that explicitly balances clinical improvement with intervention cost. This formulation solves a Lagrangian-constrained optimization problem, bridging the gap between high-level medical reasoning and precise temporal credit assignment in critical care.\\

\section{Analysis of Tri-Drive Fitness}
To assess the reward functions before training, we utilize the three proposed fitness objectives: Survival ($J_{\text{surv}}$), Confidence ($J_{\text{conf}}$), and Competence ($J_{\text{comp}}$). We acknowledge that baseline rewards explicitly designed with sparse outcome signals will naturally exhibit high $J_{\text{surv}}$ scores by definition, as they directly encode the survival target. Consequently, a superior reward function should not merely maximize $J_{\text{surv}}$ but demonstrate a robust balance across all three Pareto dimensions, particularly in Competence ($J_{\text{comp}}$) and Confidence ($J_{\text{conf}}$). We use J in both reward selection and evaluation since it is an off-policy evaluation functional of a fixed dataset and a fixed clinician policy, not a training objective. Reward selection and policy training are decoupled in our setting. We treat reward design as a model selection problem, and J as a task-level performance criterion that is external to the learning dynamics.\\
\begin{figure}[htbp]
    \centering
    \includegraphics[width=\linewidth]{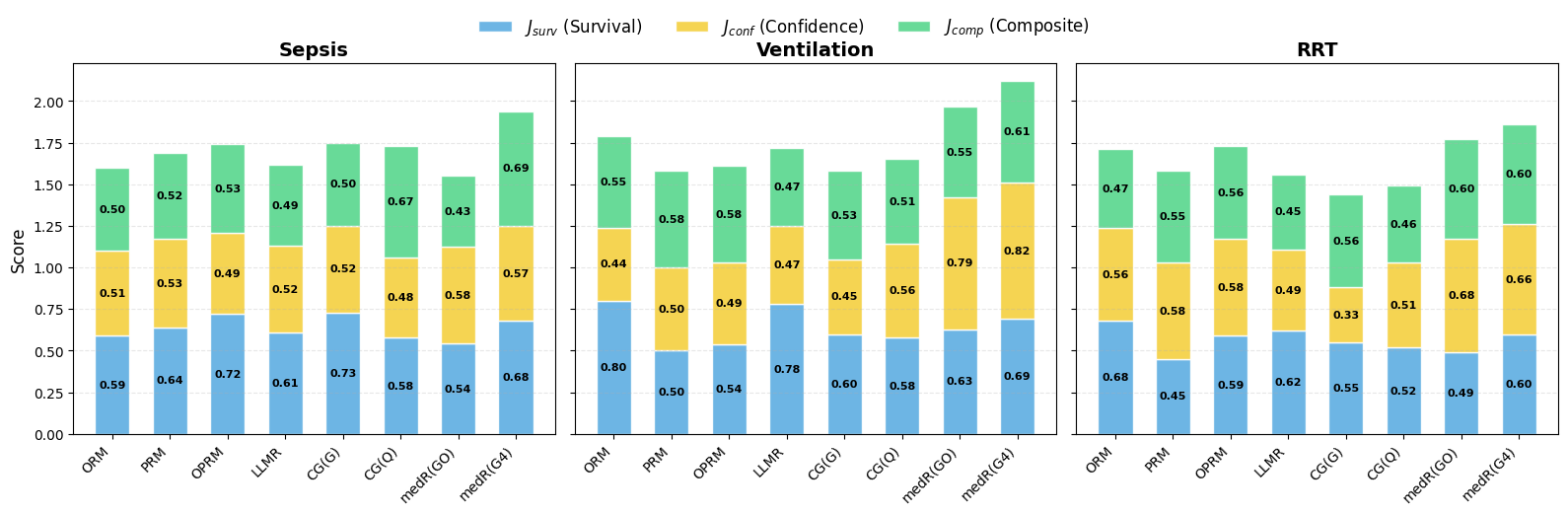}
    \caption{Decomposition of Tri-Drive fitness across tasks. The stacked bars represent the contribution of three distinct evaluation metrics: Survival ($J_{surv}$), Confidence ($J_{conf}$), and Competence ($J_{comp}$).}
    \label{fig:j_bar}
\end{figure}

\section{Features and Actions}
\label{sec:feature}

\subsection{Sepsis}
The sepsis cohort utilizes 46 observation features, comprising demographics, vital signs, and a comprehensive set of laboratory values, as detailed in Table \ref{table:sepsis_features}. For the action space, we consider two primary interventions: IV fluids and Vasopressors. The dosage for each intervention is discretized into four quantile bins plus a no-treatment option, resulting in 5 discrete levels per drug. The joint action space is the Cartesian product of these levels, yielding \(5 \times 5 = 25\) discrete treatment combinations.

\begin{table}[htbp]
\centering
\small
\caption{List of features for the Sepsis task.}
\begin{tabular}{p{2.5cm}p{10.5cm}}
\toprule
\textbf{Category} & \textbf{Feature Name} \\ \midrule
Demographics (6) &
  Age, Gender, Weight, Readmission, Elixhauser Score, Mechanical Ventilation Status \\ \midrule
Vital Signs (8) &
  Heart Rate, Respiratory Rate, SpO\textsubscript{2}, Temperature, SBP, DBP, MBP, Shock Index \\ \midrule
Lab Values (30) &
  Lactate, PaO\textsubscript{2}, PaCO\textsubscript{2}, pH, Base Excess, CO\textsubscript{2}, Hemoglobin, PaO\textsubscript{2}/FiO\textsubscript{2}, WBC, Platelet, BUN, Creatinine, PTT, PT, INR, AST, ALT, Bilirubin, Magnesium, Ionized Calcium, Calcium, Urine Output, Potassium, Sodium, Chloride, Glucose, Albumin, Bicarbonate, FiO\textsubscript{2} \\ \midrule
Scores (2) &
  SOFA (24hr), SIRS, GCS \\ \bottomrule
\end{tabular}
\begin{flushleft}
\scriptsize Abbreviations - SBP/DBP/MBP: Systolic/Diastolic/Mean Blood Pressure; GCS: Glasgow Coma Scale; BUN: Blood Urea Nitrogen; WBC: White Blood Cell; PTT: Partial Thromboplastin Time; PT: Prothrombin Time; INR: International Normalized Ratio; AST: Aspartate Aminotransferase; ALT: Alanine Aminotransferase.
\end{flushleft}
\label{table:sepsis_features}
\end{table}

\subsection{Mechanical Ventilation (MV)}
For the ventilation task, we utilize 41 features selected for training, listed in Table \ref{table:mv_features}. The action space consists of 18 discrete combinations derived from three ventilator settings: Positive End-Expiratory Pressure (PEEP), Fraction of Inspired Oxygen (FiO\textsubscript{2}), and Tidal Volume adjusted for ideal body weight. These are categorized into Low, Medium, and High levels as summarized in Table \ref{tab:mv_action_levels}.

\begin{table}[htbp]
\centering
\small
\caption{List of features for the Mechanical Ventilation task.}
\begin{tabular}{p{2.5cm}p{10.5cm}}
\toprule
\textbf{Category} & \textbf{Feature Name} \\ \midrule
Demographics (4) &
  Age, Gender, Weight, Elixhauser Score \\ \midrule
Vital Signs (7) &
  Heart Rate, Respiratory Rate, SpO\textsubscript{2}, Temperature, SBP, DBP, MBP \\ \midrule
Lab Values (26) &
  Lactate, Bicarbonate, PaCO\textsubscript{2}, pH, Base Excess, Chloride, Potassium, Sodium, Glucose, Hemoglobin, Ionized Calcium, Calcium, Magnesium, Albumin, BUN, Creatinine, WBC, Platelet, PTT, PT, INR, ETCO\textsubscript{2}, Urine Output, Rate Std, Analgesic/Sedative Admin, Neuromuscular Blocker Admin \\ \midrule
Scores (4) &
  GCS, SOFA (24hr), SIRS (24hr), RASS \\ \bottomrule
\end{tabular}
\label{table:mv_features}
\end{table}

\begin{table}[h!]
\centering
\caption{Categorization of action levels for PEEP, FiO\textsubscript{2} and Tidal Volume in the MV task.}
\begin{tabular}{ccc}
\toprule
\textbf{Intervention} & \textbf{Category} & \textbf{Threshold}\\
\midrule
\multirow{2}{*}{PEEP (cmH\textsubscript{2}O)} & Low     & $\leq 5$ \\
                                              & High    & $> 5$  \\
\midrule
\multirow{3}{*}{FiO\textsubscript{2} (\%)}    & Low     & $< 35$    \\
                                              & Medium  & $35$–$50$   \\
                                              & High    & $\geq 50$  \\
\midrule
\multirow{3}{*}{Tidal Volume (ml/kg IBW)}     & Low     & $< 6.5$  \\
                                              & Medium  & $6.5$–$8$  \\
                                              & High    & $\geq 8$  \\
\bottomrule
\end{tabular}
\label{tab:mv_action_levels}
\end{table}

\subsection{Renal Replacement Therapy (RRT)}
The RRT task utilizes a specific subset of 18 features focusing on renal function and hemodynamic stability, as shown in Table \ref{tab:rrt_features}. Unlike the multi-dimensional discrete actions in Sepsis and Ventilation, the RRT action space is defined as a single dimension representing the aggregate dosage of renal replacement therapy.

\begin{table}[!htbp]
\centering
\small
\caption{List of features for the RRT task.}
\resizebox{\linewidth}{!}{
\begin{tabular}{p{2.5cm}p{10.5cm}}
\toprule
\textbf{Category} & \textbf{Feature Name} \\ \midrule
Demographics (3) &
  Age, Gender, Weight \\ \midrule
Vital Signs (5) &
  Heart Rate, SBP, MBP, DBP, SpO\textsubscript{2} \\ \midrule
Lab Values (8) &
  Urea, Creatinine, pH, Bicarbonate, Potassium, Sodium, PaO\textsubscript{2}/FiO\textsubscript{2}, SOFA (24hr) \\ \midrule
Urine Output (2) &
  Total Urine Output (6hr), Urine Output Change Rate (6hr) \\ \bottomrule
\end{tabular}}
\label{tab:rrt_features}
\end{table}

\section{Behavior Policy}
To approximate the clinician's decision-making process and facilitate Off-Policy Evaluation (OPE), we developed a behavior policy parameterized by a Multi-Layer Perceptron (MLP). This policy was trained via supervised learning to imitate observed clinical actions. Across all tasks, models were trained for 100 epochs, with architectures and hyperparameters optimized for each specific domain. For \textbf{Sepsis}, we employed a 3-layer network with 32 hidden units (batch size 256, learning rate $5 \times 10^{-4}$). The \textbf{Ventilation} task utilized a compact 2-layer architecture with 8 hidden units (batch size 128, learning rate $1 \times 10^{-3}$), while the \textbf{RRT} model consisted of 3 layers with 16 hidden units (batch size 128, learning rate $5 \times 10^{-4}$).

\section{Pseudocodes}
\begin{algorithm}[ht]
\caption{Reward Engineering via Tri-Drive Potential Functions}
\label{alg:reward}
\begin{algorithmic}
   \STATE {\bfseries Input:} Clinical Dataset $\mathcal{D}$, LLM $\mathcal{M}$, Candidate Features $\mathcal{F}$
   \STATE {\bfseries Output:} Pareto-Optimal Set of Reward Functions $\mathcal{R}^*$
   
   \vspace{0.1cm}
   \STATE \textcolor{gray}{// Phase 1: Interpretable Feature Selection}
   \STATE Calculate statistical metadata $M_f$ for all $f \in \mathcal{F}$ (missingness, correlation $\rho$)
   \STATE $F_{crit} \leftarrow \text{EnsembleSelect}(\mathcal{M}(M_f, \text{Prompt}_{feat}))$ 
   
   \vspace{0.1cm}
   \STATE \textcolor{gray}{// Phase 2: Candidate Reward Function Generation}
   \STATE Initialize Candidate Set $\mathcal{C} \leftarrow \emptyset$
   \FOR{$i = 1$ to $N$}
       \STATE $r_i \leftarrow \text{ExtractCode}(\mathcal{M}(\text{Prompt}_{gen}, F_{crit}))$
       \STATE Add $r_i$ to $\mathcal{C}$
   \ENDFOR
   
   \vspace{0.1cm}
   \STATE \textcolor{gray}{// Phase 3: Multi-Objective Selection}
   \STATE For all $r \in \mathcal{C}$, compute Offline Fitness Vector:
   \STATE \quad $\mathbf{J}(\Phi) = [J_{surv}, J_{conf}, J_{comp}]$ 
       
   
   \STATE \textbf{Finalize:} $\mathcal{R}^* \leftarrow \text{ParetoFront}(\mathcal{C})$
\end{algorithmic}
\end{algorithm}
\section{Comparison of Action Distributions between RL and Clinicians}
\begin{table}[htbp]
\centering
\caption{Agreement rates between AI policy and Clinician actions across three clinical tasks. Values represent Mean $\pm$ Standard Deviation across seeds. The **Joint** column represents the exact match rate across all action dimensions combined. \textbf{Bold} indicates the highest agreement in each category.}
\label{tab:action_agreement}

\resizebox{\textwidth}{!}{%
\begin{tabular}{l ccc cccc c}
\toprule
\multirow{2}{*}{\textbf{Method}} & \multicolumn{3}{c}{\textbf{Sepsis Agreement (\%)}} & \multicolumn{4}{c}{\textbf{Ventilation Agreement (\%)}} & \textbf{RRT (\%)} \\
\cmidrule(lr){2-4} \cmidrule(lr){5-8} \cmidrule(lr){9-9}
& \textbf{Joint} & \textbf{IV} & \textbf{Vaso} & \textbf{Joint} & \textbf{PEEP} & \textbf{FiO2} & \textbf{Tidal} & \textbf{Dose} \\
\midrule

ORM 
& $23.99 \pm 0.40$ & $29.06 \pm 0.23$ & $75.57 \pm 0.89$ 
& $21.28 \pm 2.20$ & $72.97 \pm 0.36$ & $51.66 \pm 5.10$ & $47.28 \pm 0.40$ 
& $54.39 \pm 9.23$ \\

PRM 
& $25.82 \pm 0.55$ & $31.62 \pm 0.71$ & $76.34 \pm 0.65$ 
& $22.17 \pm 2.72$ & $75.29 \pm 1.64$ & $\mathbf{55.29 \pm 5.48}$ & $46.06 \pm 0.50$ 
& $52.54 \pm 11.33$ \\

OPRM 
& $25.50 \pm 0.54$ & $31.13 \pm 0.60$ & $76.34 \pm 0.81$ 
& $20.81 \pm 2.50$ & $73.01 \pm 0.68$ & $54.63 \pm 2.65$ & $42.88 \pm 1.79$ 
& $60.43 \pm 3.12$ \\

\midrule

CodeGen$_{\texttt{+LLM}}$ 
& $25.72 \pm 0.38$ & $31.80 \pm 0.40$ & $76.29 \pm 0.73$ 
& $17.43 \pm 4.15$ & $68.40 \pm 0.32$ & $50.00 \pm 7.55$ & $38.18 \pm 3.78$ 
& $60.14 \pm 2.47$ \\

CodeGen$_{\texttt{+GPT-4}}$ 
& $25.79 \pm 0.45$ & $31.58 \pm 0.21$ & $76.43 \pm 0.77$ 
& $22.39 \pm 1.01$ & $75.50 \pm 3.98$ & $52.42 \pm 2.93$ & $47.30 \pm 3.25$ 
& $53.30 \pm 10.30$ \\

CodeGen$_{\texttt{+Qwen}}$ 
& $25.89 \pm 0.52$ & $31.85 \pm 0.51$ & $76.45 \pm 0.70$ 
& $20.84 \pm 4.16$ & $\mathbf{75.77 \pm 1.48}$ & $54.49 \pm 4.41$ & $44.44 \pm 2.98$ 
& $54.11 \pm 9.49$ \\

\midrule

\textbf{medR$_{\texttt{+GPT-OSS-20B}}$} 
& $25.93 \pm 0.52$ & $32.21 \pm 0.47$ & $\mathbf{76.46 \pm 0.79}$ 
& $21.82 \pm 0.48$ & $71.72 \pm 0.05$ & $51.82 \pm 3.10$ & $48.03 \pm 0.55$ 
& $54.35 \pm 9.18$ \\

\textbf{medR$_{\texttt{+GPT-4}}$} 
& $\mathbf{26.02 \pm 0.33}$ & $\mathbf{32.22 \pm 0.35}$ & $\mathbf{76.46 \pm 0.74}$ 
& $\mathbf{22.59 \pm 0.47}$ & $71.53 \pm 0.41$ & $51.50 \pm 3.64$ & $\mathbf{48.20 \pm 0.78}$ 
& $\mathbf{65.99 \pm 2.46}$ \\

\bottomrule
\end{tabular}%
}
\end{table}

\begin{figure}[h]
    \centering
    \includegraphics[width=\linewidth]{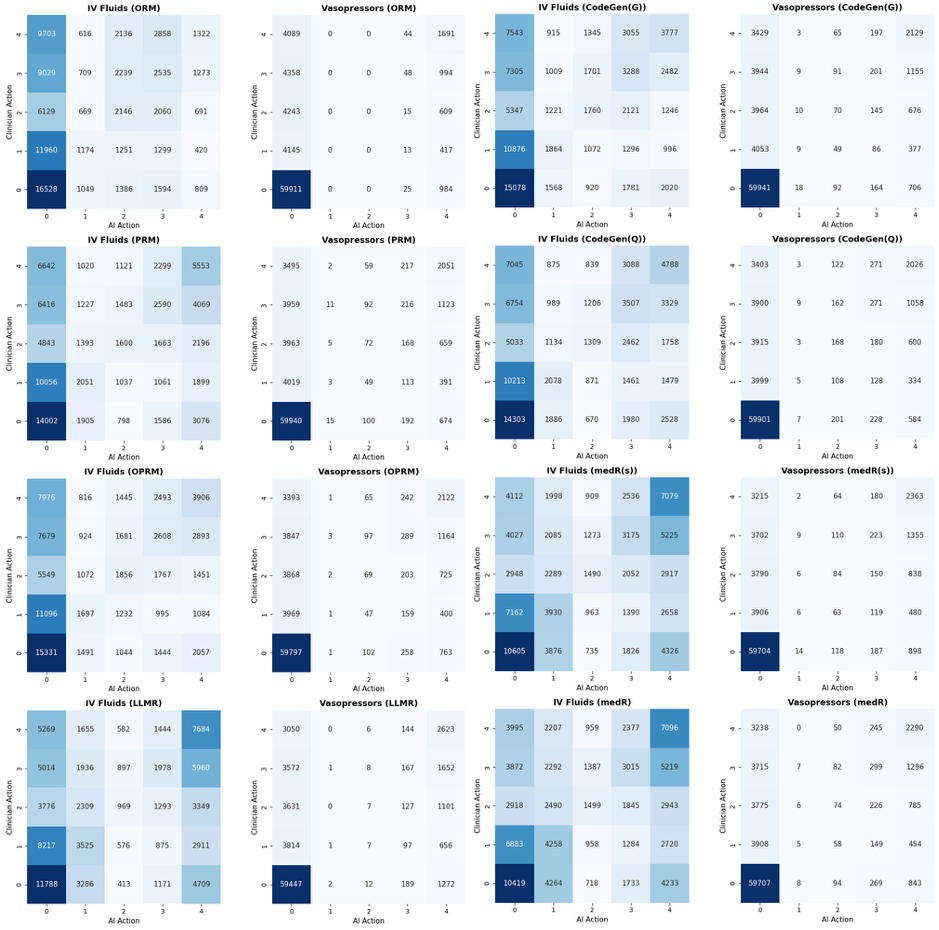} 
    \caption{\textbf{Sepsis Policy Agreement.} Confusion matrices comparing AI vs. Clinician actions for IV Fluids (left) and Vasopressors (right) across different reward functions.}
    \label{fig:sepsis_agreement}
\end{figure}
\begin{figure}[h]
    \centering
    \includegraphics[width=\linewidth]{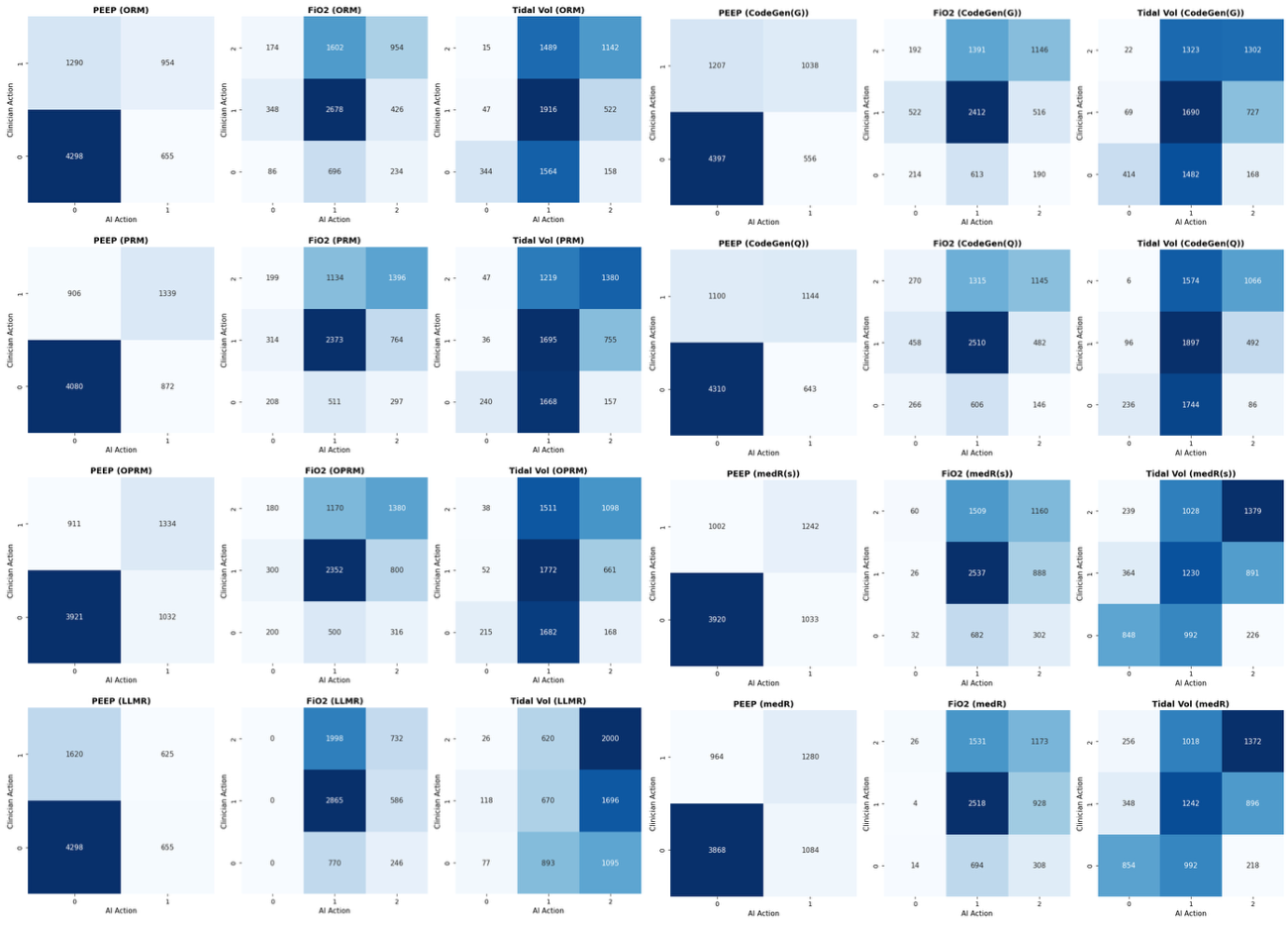} 
    \caption{\textbf{Ventilation Policy Agreement.} Comparison of discrete action choices for PEEP, FiO$_2$, and Tidal Volume. Darker diagonals indicate higher concordance with clinical standards.}
    \label{fig:ventilation_agreement}
\end{figure}
\begin{figure}[h]
    \centering
    \includegraphics[width=\linewidth]{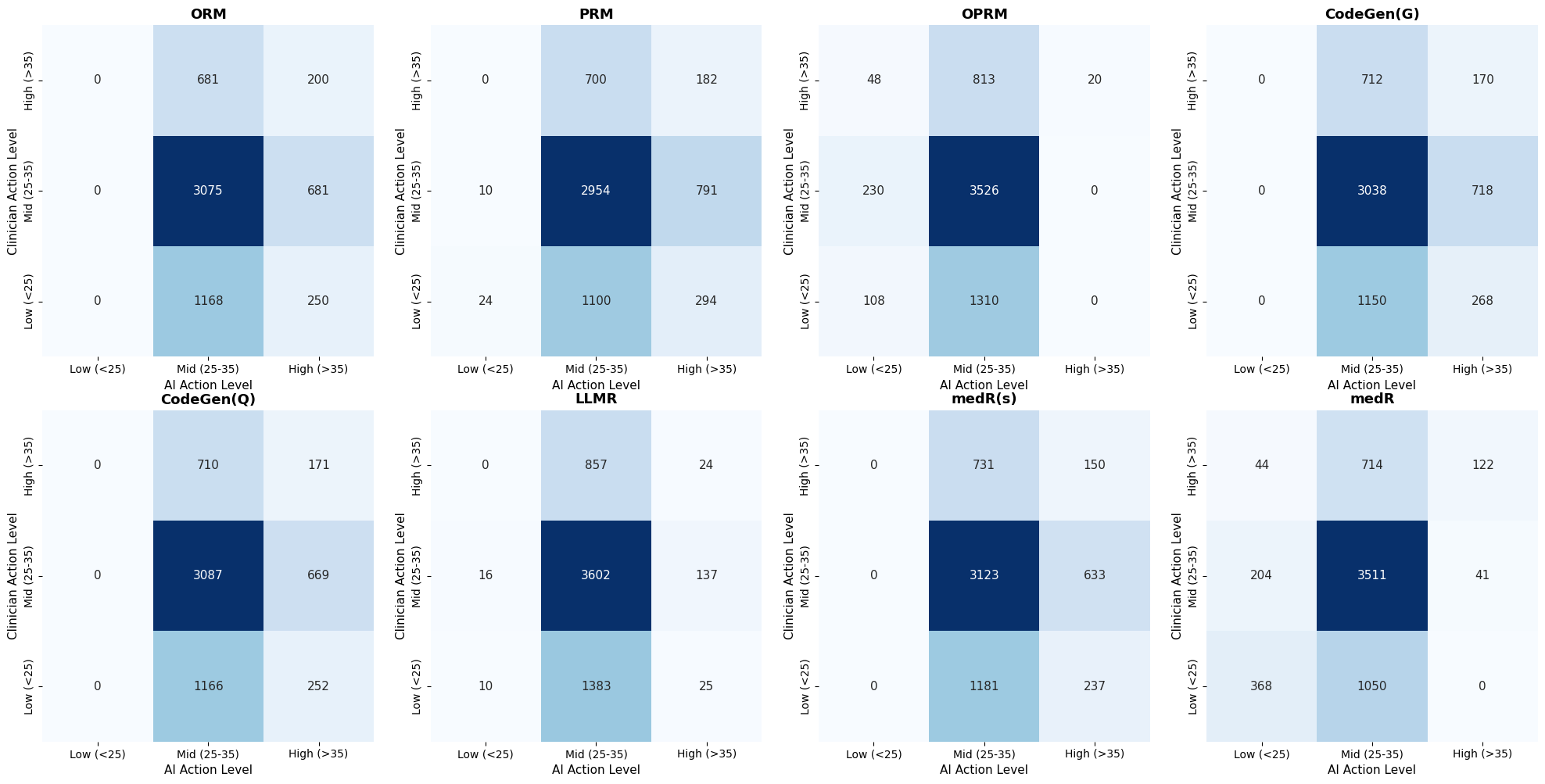} 
    \caption{\textbf{RRT Dosing Agreement.} Analysis of continuous dosing converted to discrete clinical bins (Low, Mid, High). The proposed method shows stronger alignment in the 'Mid' (Standard) dosing range.}
    \label{fig:rrt_agreement}
\end{figure}
\end{document}